\documentclass[letterpaper]{article} 
\usepackage[preprint]{aaai2027} 
\usepackage[hyphens]{url} 
\usepackage{graphicx} 
\usepackage{natbib} 
\usepackage{caption} 
\usepackage{amsmath,amssymb,amsthm}
\usepackage{algorithm}
\usepackage{algpseudocode}
\usepackage{booktabs}
\usepackage{multirow}

\makeatletter
\def\fps@figure{!t}
\def\fps@table{!t}
\makeatother

\newtheorem{proposition}{Proposition}
\newcommand{\method}{\textsc{ENOUGH}}
\newcommand{\stopa}{\textsc{Stop}}
\newcommand{\pos}[1]{\left[#1\right]_{+}}
\newcommand{\concat}{\mathbin{\Vert}}
\newcommand{\efficiencytablelabel}{tab:efficiency}

\title{Less Is Personal: Learning Minimal Sufficient User Profiles\\
for Personalized Language Models}
\author{
Minghang Liu$^{\spadesuit,\heartsuit}$,
Qiang Qiu$^{\spadesuit}$,
Yuanzhuo Wang$^{\spadesuit}$\\
Huawei Shen$^{\spadesuit}$,
Xueqi Cheng$^{\spadesuit}$
}
\affiliations{
\Large
$^{\spadesuit}$State Key Laboratory of AI Safety, Institute of Computing Technology, CAS\\
$^{\heartsuit}$University of Chinese Academy of Sciences, Beijing, China\\
\fontsize{10}{14}\selectfont
\{liuminghang23s, wangyuanzhuo, qiangqiu, shenhuawei\}@ict.ac.cn\\[6pt]
}

\begin{document}
\defcitealias{chen-etal-2025-icr}{Chen et al.\ 2025}
\maketitle

\begin{abstract}

Retrieval-augmented personalization enables large language models to produce more accurate and preference-aligned outputs using relevant records retrieved from user histories. Personalized language models typically prepend a fixed number of retrieved user records, even when additional history is redundant, harmful, or unrelated to a user's distinctive behavior. We study minimal sufficient personalization: constructing the least costly ordered profile for each input while preserving the utility achievable from a retrieved candidate pool. We introduce ENOUGH, a method that iteratively appends behavioral records or emits STOP to construct profiles with adaptive lengths. Offline, bounded counterfactual search evaluates profile prefixes by jointly considering downstream gains, user specificity, and token costs. The resulting long-horizon targets are distilled into a multi-head value controller with explicit ranking and stopping supervision. At inference, the controller selects and orders records through lightweight decisions, and the frozen generator is invoked once after stopping. Extensive experiments on six personalized tasks demonstrate that ENOUGH consistently outperforms strong heuristic and retrieval-augmented baselines in both effectiveness and efficiency, achieving minimal sufficient profiles that preserve personalization utility while reducing unnecessary context costs.

\end{abstract}

\section{Introduction}
\label{sec:introduction}

Personalized large language models (LLMs) adapt predictions and generations to a user's preferences, expertise, and behavioral patterns, supporting tasks from classification and recommendation to tailored text generation \citep{salemi-etal-2024-lamp,zhang-etal-2025-personalization-survey}. One approach encodes the user in model parameters or a static textual profile \citep{zhang-2024-guided,tan-etal-2024-democratizing}. Such persistent representations, however, can be costly to update or may compress away the fine-grained, request-dependent evidence in a growing interaction history. Retrieval-augmented personalized LLMs offer a flexible approach: for each request, they retrieve a set of the user's past behavioral records and place them in the prompt, allowing a frozen generator to adapt without per-user retraining \citep{salemi-etal-2024-lamp,xu-etal-2025-personalized}.

However, making history accessible does not determine how it should be used. A typical retrieval-augmented pipeline ranks records by semantic relevance and prepends a fixed top-$k$ profile. Recent methods improve this pipeline through generation feedback \citep{salemi-etal-2024-optimization}, attention-derived relevance scoring \citepalias{chen-etal-2025-icr}, or order-aware profile policies \citep{du-etal-2026-optimizing}, but the resulting profile is still commonly a fixed-size set or fixed-length permutation. This view leaves two questions unresolved: \emph{how much history should be used?} and \emph{which useful history is genuinely personal?}

\begin{figure}[!t]
    \centering
    \includegraphics[width=\columnwidth]{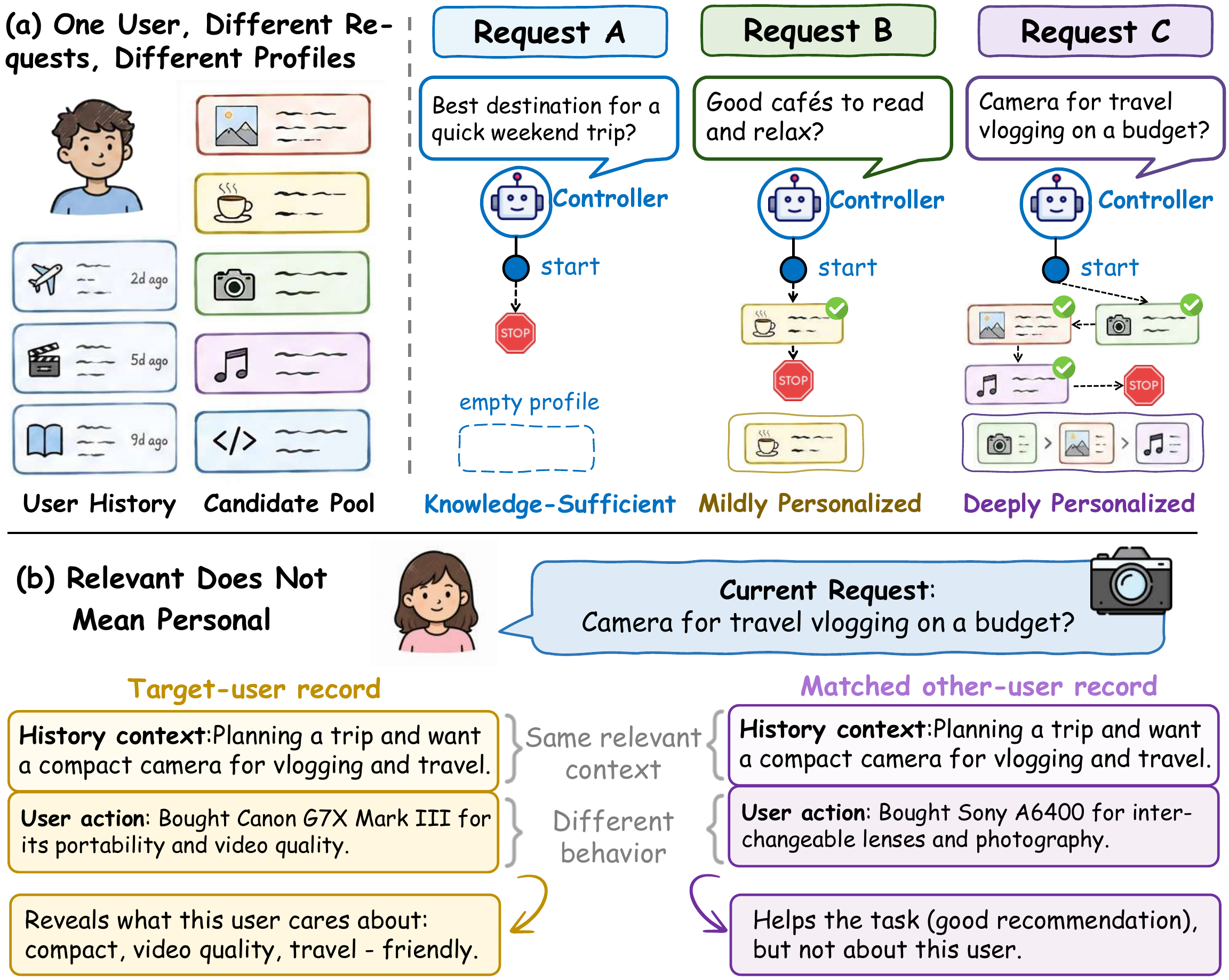}
    \caption{
    (a) A user's requests may require profiles of different lengths, including an empty one.
    (b) Query-relevant records can encode different user actions, so usefulness need not imply target-user specificity.}
    \label{fig:motivation}
\end{figure}

\textit{How much history should be used?}
The useful amount of history depends on both the current request and the records already selected. As illustrated in Figure~\ref{fig:motivation}(a), a request may be answerable without personal context, may need a single revealing behavior, or may require several selected records. A fixed profile length offers no mechanism to recognize when later records merely repeat an established preference, introduce conflicting signals, or dilute useful context. Personalized generation should thus be viewed as a selective, self-terminating evidence-acquisition process, not an always-on retrieval regime. Model should decide whether history is needed at all and identify when further history ceases.

\textit{Which useful history is genuinely personal? }Relevance, general task utility, and personalization answer different questions. Semantic similarity indicates whether a historical context concerns the current request, but not whether the associated action reveals the target user's behavior. Likewise, a record can improve generation simply by serving as a generally useful in-context example. Figure~\ref{fig:motivation}(b) illustrates this distinction: target-user and other-user records share the same relevant context, yet their actions imply different preferences. 
We operationalize personal value through two contrasts. The no-profile contrast asks whether the complete selected profile helps at all. For a beneficial profile, content-matched other-user replacements then ask whether a target-user record contributes more than a comparably relevant record while the remaining profile is held fixed. Thus, for our frozen generator, history is valued neither by semantic similarity nor by generic generation gain, but by its \emph{user-specific marginal value} under the current profile.

These two questions are inseparable. The marginal value of a candidate changes with the selected profile: a record that is useful initially can become redundant after a similar behavior has been included, while records may interact or complement one another. Consequently, history length should not be a hyperparameter chosen independently of record value. It should emerge from the same conditional decision process: continue while some candidate offers future personalized value, and stop when the current profile is sufficient. We call this objective \emph{minimal sufficient personalization}: within a retrieved candidate pool, seek the least costly ordered profile that approaches its best attainable personalized utility.

We propose Efficient Net-utility Optimization for User-profile Growth and Halting (\method{}), a method that targets this objective by either appending a behavioral record or emitting \stopa{}.
Starting from the empty profile, \method{} can reject profile augmentation at the root, choose and order interacting records, and terminate at any feasible length. 
Offline, bounded counterfactual search evaluates profile prefixes through three quantities: downstream gain over no personalization, beneficial user specificity measured against content-matched other user replacements, and actual incremental prompt-token cost. 
Long-horizon backups then supervise a multi-head value controller, so selection and stopping account for future record interactions rather than only immediate gains. 
At inference, the controller constructs the profile through lightweight decisions, and the frozen generator is invoked once after stopping.
Our contributions are threefold:
\begin{itemize}
    \item We formulate retrieval-augmented personalization as minimal sufficient profile construction, unifying whether to personalize, which behavioral records to select, and when additional history is no longer worthwhile.
    \item We introduce a dual counterfactual utility that separates general downstream gain from user-specific marginal value, and a long-horizon sequential controller with an explicit \stopa{} action. 
    \item Our evaluation spans six tasks across a broad range of personalization settings, including classification, regression, and text generation, together with diagnostics of stopping, specificity, robustness, and efficiency.
\end{itemize}

\section{Methodology}
\label{sec:method}




As shown in Figure~\ref{fig:method-overview}, \method{} constructs an input-adaptive ordered user profile that targets minimal sufficiency from the top-$M$ candidates of a frozen retriever. Starting from the empty profile, a lightweight controller sequentially appends a behavioral record consisting of its historical context and the user's action or selects \stopa{}. Offline, bounded counterfactual search evaluates profiles using downstream gain, target-user specificity, and token cost, then distills long-horizon action values into the controller. At test time, construction uses neither references nor generator rollouts; the frozen generator is invoked once after stopping.

\begin{figure*}[!t]
    \centering
    \includegraphics[width=\textwidth]{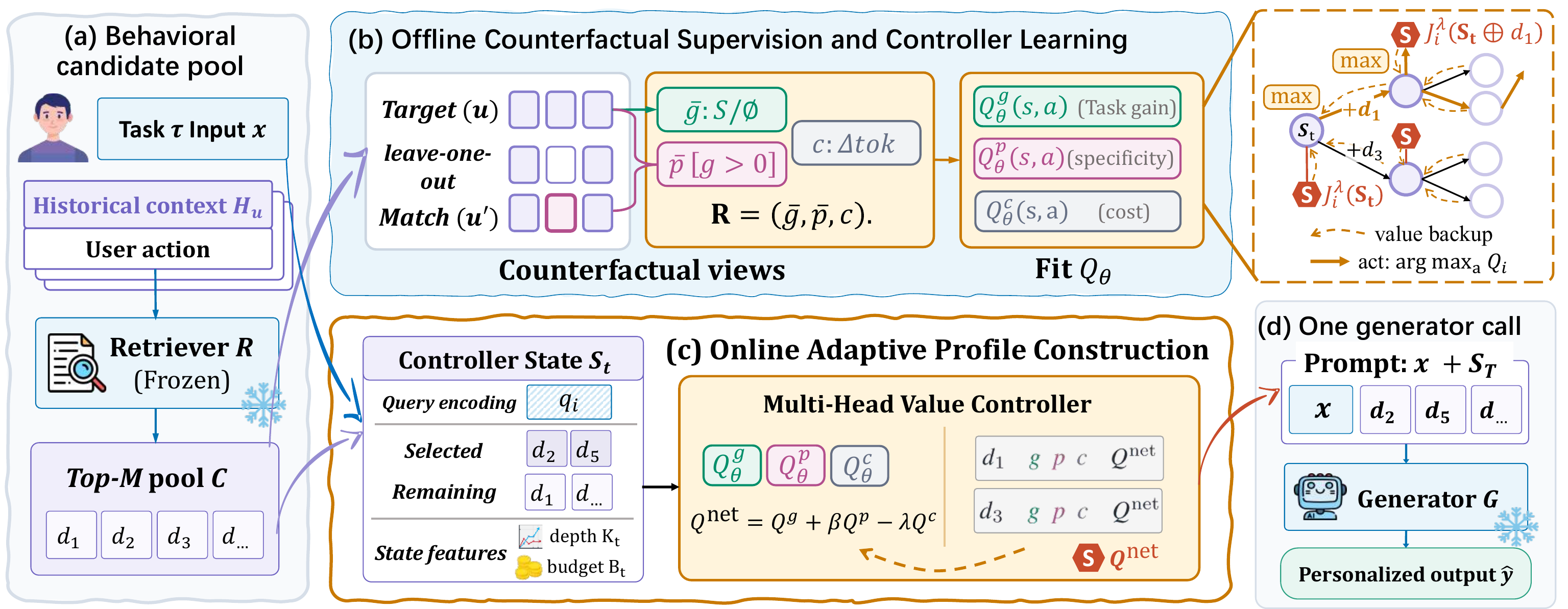}
    \vspace{-1.5em}
    \caption{Overview of the ENOUGH framework.
    (a) A frozen retriever forms a behavioral candidate pool.
    (b) Offline counterfactual views measure downstream gain, matched-replacement specificity, and prompt cost; value backup supplies aligned long-horizon targets for the controller. 
    (c) At test time, the controller adaptively scores feasible record additions and \stopa{} by net value.
    (d) Once construction stops, the ordered profile and current input are passed to the frozen generator in one call.
    }
    \label{fig:method-overview}
    \vspace{-1.5em}
\end{figure*}

\subsection{Problem Formulation}
\label{sec:problem}

Let the training set be
\begin{equation}
\small
    \mathcal D =
    \{(u_i,\tau_i,x_i,y_i,H_{u_i})\}_{i=1}^{N},
\end{equation}
where $u_i$ denotes a profile owner, $\tau_i$ its personalization task, $x_i$ the current input, $y_i$ its reference output, and $H_{u_i}$ the  history. Each history item is $d=(c_d,a_d,m_d)$, with $c_d$ is the historical context, $a_d$ is an observed user action (e.g., a label, score, or generated text), and $m_d$ contains optional metadata such as timestamps. A task-specific serializer $\phi_{\tau}$ maps each record into a common textual schema.

A frozen retriever $R$ constructs a candidate pool
\begin{equation}
\small
    C_i=R(x_i,H_{u_i};M), \qquad |C_i|\leq M,
\end{equation}
where $M$ is the candidate-pool size limit and we set $    K_i=\min(K_{\max},|C_i|)$.

We consider the space $\mathfrak S_i$ of ordered, non-repeating sequences from $C_i$ whose length is at most $K_i$ and whose rendered prompt satisfies a hard context limit $L_{\max}$. At step $t$ the controller holds $\mathbf S_t=(d_1,\dots,d_t)$ and chooses
\begin{equation}
\small
    a_t\in\mathcal A(s_t)
    =\{d\in C_i\setminus\mathbf S_t:
      \mathbf S_t\oplus d\in\mathfrak S_i\}
      \cup\{\stopa\}.
\end{equation}

Selecting $d$ appends it to the ordered profile; selecting \stopa{} terminates construction. The final output is
\begin{equation}
\small
    \widehat y_i=G\!\left(P_{\tau_i}(x_i,\mathbf S_T)\right),
\end{equation}
where $G$ is the frozen generator and $P_{\tau}$ is prompt template.

\paragraph{Minimal sufficient personalization.}
Let $U_i(\mathbf S)$ denote the cost-free personalized utility defined in Section~\ref{sec:dual-value}, and let
\begin{equation}
\small
    U_i^{\star}=\max_{\mathbf S\in\mathfrak S_i}U_i(\mathbf S).
\end{equation}
An $\epsilon$-minimal sufficient profile is
\begin{equation}
\small
    \mathbf S_i^{\mathrm{MSP}}
    \in
    \arg\min_{\mathbf S\in\mathfrak S_i} c_i(\mathbf S)
    \quad \text{s.t.}\quad
    U_i(\mathbf S)\geq U_i^{\star}-\epsilon.
    \label{eq:msp}
\end{equation}

This definition separates sufficiency from minimality.
Because the ordered search space is discrete and non-convex, we optimize the cost-regularized surrogate: 
\begin{equation}
\small
    J_i^{\lambda}(\mathbf S)
    =U_i(\mathbf S)-\lambda c_i(\mathbf S).
    \label{eq:net-objective}
\end{equation}
Different values of $\lambda$ recover supported points on the empirical quality--cost frontier.

We measure the actual incremental prompt cost:
\begin{equation}
\small
    c_i(\mathbf S)=
    \frac{
      \operatorname{Tok}(P_{\tau_i}(x_i,\mathbf S))
      -\operatorname{Tok}(P_{\tau_i}(x_i,\varnothing))
    }{L_{\mathrm{ref}}},
    \label{eq:cost}
\end{equation}
where $L_{\mathrm{ref}}$ is a profile reference budget shared by all tasks. 


\subsection{Behavior-Aware Profile Representation}
\label{sec:representation}

Personalization histories are observations of user behavior. The controller therefore encodes both the situation and the user's response. For a candidate record $d_j$, we compute
\begingroup
\fontsize{8.5pt}{10pt}\selectfont
\begin{equation}
    \begin{aligned}
    e_j=E_h\!\big(&[\mathrm{TASK}=\tau_i]\concat
      [\mathrm{HIST\_CONTEXT}]\concat c_{d_j}\\
      &{}\concat[\mathrm{USER\_ACTION}]\concat a_{d_j} \concat[\mathrm{META}]\concat m_{d_j}\big).
    \end{aligned}
    \label{eq:record-encoder}
\end{equation}
\endgroup

The current request is encoded as $q_i=E_q([\mathrm{TASK}=\tau_i]\concat[\mathrm{QUERY}]\concat x_i)$. 

The controller state induced by the ordered profile $\mathbf S_t$ is
\begin{equation}
    s_t=(q_i,z_t,r_t,b_t),
\end{equation}
where $z_t=E_S(e_{d_1},\ldots,e_{d_t})$ is an order-sensitive incremental representation of the selected profile. $r_t=\operatorname{AttnPool}([q_i;z_t],\{e_j:d_j\in C_i\setminus\mathbf S_t\})$ is a permutation-invariant summary of the remaining candidates. $b_t$ contains the current incremental token cost, remaining horizon, and remaining context budget. In the multi-task setting, the encoders and value controller are shared across personalization tasks, while the task embedding, serializer, and prompt template are task-conditioned.

\subsection{Dual Counterfactual Profile Utility}
\label{sec:dual-value}

\subsubsection{Downstream task gain}

For a user profile $\mathbf S$, we compute the length-normalized log-likelihood of the reference output:
\begin{equation}
    \ell_i(\mathbf S)
    =\frac{1}{|y_i|}
      \sum_{r=1}^{|y_i|}
      \log p_G\!\left(
        y_{i,r}\mid y_{i,<r},P_{\tau_i}(x_i,\mathbf S)
      \right).
    \label{eq:likelihood}
\end{equation}
The downstream gain over no profile is
\begin{equation}
\small
    g_i(\mathbf S)=\ell_i(\mathbf S)-\ell_i(\varnothing).
    \label{eq:task-gain}
\end{equation}

To prevent high variance tasks from dominating shared training, we use a scale-only normalization. On a fixed calibration subset of the initial offline branches, we compute
\begin{equation}
    s_{\tau}^{g}
    =\sqrt{\mathbb E_{i:\tau_i=\tau}[g_i(\mathbf S)^2]},
    \qquad
    \overline g_i(\mathbf S)
    =\frac{g_i(\mathbf S)}{s_{\tau_i}^{g}+\varepsilon}.
    \label{eq:g-normalization}
\end{equation}
We do not subtract a mean, preserving $\overline g_i(\varnothing)=0$.

\subsubsection{User-specific gain over matched replacements}

A profile can improve the likelihood simply because it contains a useful in-context demonstration. To isolate user-specific usefulness, we intervene on one profile slot while holding the rest of the selected profile fixed.

For every candidate $d$ owned by $u_i$, we precompute $K^{-}$ matched replacements $m_k(d;u_i)$ from other profile owners. Matching uses only variables available before observing the user action: task identity, historical context semantics, relevance to the current query, length, and time when present.

For a profile $\mathbf S=(d_1,\ldots,d_t)$ and position $r$, let $\mathbf S_{-r}$ be the sequence with $d_r$ removed while preserving the order of all other records. 
For each slot, the target-user and matched-control marginal gains are then
\begin{align}
\small
    \Delta^u_{i,r}(\mathbf S)
    &=\ell_i(\mathbf S)-\ell_i(\mathbf S_{-r}),\\
    \Delta^{c,k}_{i,r}(\mathbf S)
    &=\ell_i(\mathbf S_{-r}\oplus m_k(d_r;u_i))
      -\ell_i(\mathbf S_{-r}).
\end{align}

The slot-level beneficial specificity is
\begin{equation}
\small
    \rho_{i,r}(\mathbf S)
    =\pos{\Delta^u_{i,r}(\mathbf S)}
     -\frac{1}{K^{-}}
      \sum_{k=1}^{K^{-}}\pos{\Delta^{c,k}_{i,r}(\mathbf S)}.
    \label{eq:slot-specificity}
\end{equation}
We credit specificity only when the complete target-user profile improves over the empty profile:
\begin{equation}
\small
    p_i(\mathbf S)
    =\mathbb{I}[g_i(\mathbf S)>0]\,
      \frac{1}{t}\sum_{r=1}^{t}\rho_{i,r}(\mathbf S),
    \qquad p_i(\varnothing)=0.
    \label{eq:profile-specificity}
\end{equation}
This prevents a profile that is harmful relative to no personalization from receiving a positive specificity bonus merely because the matched controls are even worse. A single uniformly sampled position per profile yields an unbiased Monte Carlo estimator of Eq.~\eqref{eq:profile-specificity}.

As with task gain, we compute a scale on the fixed calibration subset and freeze it:
\begin{equation}
\small
    s_{\tau}^{p}
    =\sqrt{\mathbb E_{i:\tau_i=\tau}[p_i(\mathbf S)^2]},
    \qquad
    \overline p_i(\mathbf S)
    =\frac{p_i(\mathbf S)}{s_{\tau_i}^{p}+\varepsilon}.
\end{equation}
The cost-free personalized utility objective is
\begin{align}
\small
    U_i(\mathbf S)
      &=\overline g_i(\mathbf S)+\beta\overline p_i(\mathbf S).
      \label{eq:utility}
\end{align}
Here, $\overline g_i(\mathbf S)$ rewards profiles that improve task likelihood over the empty profile, while $\overline p_i(\mathbf S)$ measures the additional benefit of target-user records over matched records from other users, discounting generic few-shot effects. The coefficients $\beta$ and $\lambda$ balance user specificity and prompt cost, respectively.


\subsection{Long-Horizon Decision Principle}
\label{sec:optimal-stopping}

For a candidate pool, context budget and maximum depth, the theoretical optimal value obeys the finite-horizon recursion
\begin{equation}
\small
    V_i^{\star}(\mathbf S)
    =\max\left\{
      J_i^{\lambda}(\mathbf S),
      \max_{d\in\mathcal A(s)\setminus\{\stopa\}}
      V_i^{\star}(\mathbf S\oplus d)
    \right\}.
    \label{eq:bellman}
\end{equation}
Similarly,
\begin{equation}
\small
    Q_i^{\star}(s,\stopa)=J_i^{\lambda}(\mathbf S), 
    Q_i^{\star}(s,d)=V_i^{\star}(\mathbf S\oplus d).
\end{equation}
Equation~\eqref{eq:bellman} is a Bellman characterization, not a claim that the learned controller recovers the full combinatorial optimum. A detailed proof is provided in Appendix~\ref{app:proof-exact-value-stopping}.

\begin{proposition}[Exact-value stopping]
\label{prop:exact-value-stopping}
On the finite deterministic candidate tree, if $Q_i^{\star}(s,a)$ is known for every reachable state and action, selecting an action with maximum $Q_i^{\star}$ at each step, with ties resolved in favor of \stopa{}, returns a terminal profile that maximizes $J_i^{\lambda}$ within the candidate tree.
\end{proposition}
\noindent

\subsection{Offline Counterfactual Supervision and Controller Learning}
\label{sec:offline}
\label{sec:controller}
\label{sec:learning}

The complete ordered candidate tree is intractable. We construct a bounded prefix tree $\mathcal T_i$ using offline counterfactual search and perform exact backups only within that tree. The resulting target $\widehat Q_{\mathcal T_i}$ is exact for the sampled tree but is only an approximation to $Q_i^{\star}$.

\paragraph{Bounded-tree construction.}
We first sample tree roots from the empty profile, prefixes of the frozen retriever, recency and diversity policies, and random subsets. Profile sizes are stratified from zero to $K_i$ so that the controller observes both early and late stopping states.

At each sampled root, \stopa{} is always evaluated, and all feasible remaining candidate actions are expanded. At deeper states, a fixed action budget $J$ combines high retriever-score, diversity, and uniformly sampled actions. Search retains at most $W$ prefixes per depth. Each intermediate prefix is a legal terminal node. Unique ordered prefixes are represented once in directed acyclic graph structured cache, and all teacher-forced generator evaluations are batched.

\paragraph{Long-horizon target backup.}
For an expanded state-action pair $(s,a)$, let $\mathcal L_{\mathcal T}(s,a)$ denote the terminal nodes in the sampled subtree reached after taking $a$. We select one terminal node using the same composite objective for all three value components:
\begin{equation}
\small
    L^{\star}_{\mathcal T}(s,a)
    =\arg\max_{L\in\mathcal L_{\mathcal T}(s,a)}
      J_i^{\lambda}(L).
    \label{eq:best-leaf}
\end{equation}
The vector-valued target is then $\mathbf{R}_i(s,a)=\left(R_i^g,R_i^p,R_i^c\right)$.

For \stopa{}, the leaf is the current profile. Thus, the three components always describe the same optimal continuation within the sampled tree. Define the corresponding scalar target as
$    R_i^{\mathrm{net}}(s,a)=R_i^{g}(s,a)+\beta R_i^{p}(s,a)-\lambda R_i^{c}(s,a).
$

\paragraph{Multi-head value controller.}

For every available action, the controller predicts three outcome components of the net-optimal sampled continuation: \(Q_{\theta}^{g}(s,a)\), \(Q_{\theta}^{p}(s,a)\), and \(Q_{\theta}^{c}(s,a)\). These quantities estimate normalized task gain, normalized matched-replacement specificity, and normalized prompt cost, respectively. The composite score is
\begin{equation}
\small
    Q_{\theta}^{\mathrm{net}}(s,a)
    =Q_{\theta}^{g}(s,a)
     +\beta Q_{\theta}^{p}(s,a)
     -\lambda Q_{\theta}^{c}(s,a).
    \label{eq:net-q}
\end{equation}
For any state with at least one feasible record action, define the predicted stop margin
\begin{equation}
\small
    \Delta_{\theta}^{\mathrm{stop}}(s)
    =Q_{\theta}^{\mathrm{net}}(s,\stopa{})
     -\max_{d\in\mathcal A(s)\setminus\{\stopa{}\}}
      Q_{\theta}^{\mathrm{net}}(s,d).
    \label{eq:stop-margin}
\end{equation}

\paragraph{Controller optimization.}

Let $\mathcal A_{\mathrm{lab}}(s)$ be the action subset expanded and labeled in the bounded tree. For each labeled state-action pair, we regress the three outcome components:
\begin{equation}
\small
    \mathcal L_{\mathrm{value}}
    =\mathbb E_{(s,a)}
      \sum_{k\in\{g,p,c\}}w_k
      \operatorname{Huber}\!\left(
        Q_{\theta}^{k}(s,a),R_i^{k}(s,a)
      \right).
    \label{eq:value-loss}
\end{equation}

In addition to pointwise value regression, we encourage the controller to preserve the relative preferences among the labeled actions at each state. The bounded-tree target and controller distributions over the same action subset are
\begin{align}
\small
    p^{\star}(a\mid s)
    &=\operatorname{softmax}_{a\in\mathcal A_{\mathrm{lab}}(s)}
      \left(R_i^{\mathrm{net}}(s,a)/T_r\right),\\
    p_{\theta}(a\mid s)
    &=\operatorname{softmax}_{a\in\mathcal A_{\mathrm{lab}}(s)}
      \left(Q_{\theta}^{\mathrm{net}}(s,a)/T_r\right).
\end{align}
The resulting listwise ranking loss is
\begin{equation}
\small
    \mathcal L_{\mathrm{rank}}
    =\mathbb E_s
        \left[
        D_{\mathrm{KL}}\!\left(
        p^{\star}(\cdot\mid s)
        \,\Vert\,
        p_\theta(\cdot\mid s)
        \right)
        \right].
    \label{eq:rank-loss}
\end{equation}


The stop-decision loss is applied only to fully expanded states with at least one record action, for which the labeled action mask contains every feasible action. This restriction prevents an unexpanded beneficial action from creating a false positive \stopa{} label. For such a state, define the bounded-tree teacher margin as
\begin{equation}
\small
    \Delta_{\mathcal T_i}^{\mathrm{stop}}(s)
    =R_i^{\mathrm{net}}(s,\stopa{})
      -\max_{d\in\mathcal A(s)\setminus\{\stopa{}\}}
       R_i^{\mathrm{net}}(s,d).
\end{equation}
Using the soft target $y_{\mathrm{stop}}=\sigma(\Delta_{\mathcal T_i}^{\mathrm{stop}}/T_s)$, we set
\begin{equation}
\small
    \mathcal L_{\mathrm{stop}}
    =\operatorname{BCE}\!\left(
      \sigma(\Delta_{\theta}^{\mathrm{stop}}/T_s),y_{\mathrm{stop}}
    \right).
    \label{eq:stop-loss}
\end{equation}
The complete objective is
\begin{equation}
    \mathcal L
    =\mathcal L_{\mathrm{value}}
     +\eta\mathcal L_{\mathrm{rank}}
     +\zeta\mathcal L_{\mathrm{stop}}.
    \label{eq:total-loss}
\end{equation}
Mini-batches are task-balanced by uniform task sampling.

\subsection{Online Adaptive Profile Construction}
\label{sec:online}
\label{sec:algorithms}

After offline training, only the learned controller parameters cross the offline--online boundary. Matched replacements, reference likelihoods, and counterfactual search are discarded. 
The deployment retains only $R$, the controller (with its training-time $\beta,\lambda$), and the frozen generator $G$.

The learned policy replaces the unavailable oracle value with the predicted composite score:
\begin{equation}
\small
    \pi_{\theta}(s)
    =\arg\max_{a\in\mathcal A(s)}Q_{\theta}^{\mathrm{net}}(s,a),
    \label{eq:policy}
\end{equation}
with ties resolved in favor of \stopa{}. When at least one record action remains feasible, the controller stops if and only if the predicted margin $\Delta_{\theta}^{\mathrm{stop}}(s)$ in Eq.~\eqref{eq:stop-margin} is non-negative. A \stopa{} decision at the root yields no profile augmentation, although candidate retrieval has still been performed.

\providecommand{\PHMainWins}{6 of 6}
\providecommand{\PHReplicaWins}{4 of 6}
\providecommand{\PHMainUWins}{5 of 6}
\providecommand{\PHReplicaUWins}{5 of 6}
\providecommand{\PHMainWinTasks}{LaMP-1, LaMP-2, LaMP-3, LaMP-4, LaMP-5, and LaMP-7}
\providecommand{\PHMainNonWinTasks}{none}
\providecommand{\PHReplicaWinTasks}{LaMP-1, LaMP-3, LaMP-4, and LaMP-5}
\providecommand{\PHReplicaNonWinTasks}{LaMP-2 and LaMP-7}
\providecommand{\PHMainUWinTasks}{LaMP-1, LaMP-2, LaMP-3, LaMP-4, and LaMP-7}
\providecommand{\PHMainUNonWinTasks}{LaMP-5}
\providecommand{\PHReplicaUWinTasks}{LaMP-1, LaMP-2, LaMP-3, LaMP-4, and LaMP-7}
\providecommand{\PHReplicaUNonWinTasks}{LaMP-5}
\providecommand{\PHMainFirstTaskFraction}{6 of 6}
\providecommand{\PHMainFirstTasks}{LaMP-1, LaMP-2, LaMP-3, LaMP-4, LaMP-5, and LaMP-7}
\providecommand{\PHMainNonFirstTasks}{none}
\providecommand{\PHMainAverageRank}{1.00}
\providecommand{\PHMainWinMarginMin}{0.003}
\providecommand{\PHMainWinMarginMax}{0.029}
\providecommand{\PHMainLossMarginMin}{0.000}
\providecommand{\PHMainLossMarginMax}{0.000}
\providecommand{\PHLiPTokens}{523}
\providecommand{\PHPURPLETokens}{899}
\providecommand{\PHContrieverTokens}{920}
\providecommand{\PHTokenSaving}{41.8\%}
\providecommand{\PHRootStop}{15.9\%}
\providecommand{\PHAvgRecords}{3.11}
\providecommand{\PHHarmRate}{15.8\%}
\providecommand{\PHOracleEmpty}{17.7\%}
\providecommand{\PHOracleShort}{56.3\%}
\providecommand{\PHFixedFiveHarm}{23.2\%}
\providecommand{\PHFixedTenHarm}{31.7\%}
\providecommand{\PHFixedPeakK}{5}
\providecommand{\PHFixedPeakQuality}{98.7}
\providecommand{\PHFixedPeakToTenQualityGap}{0.8}
\providecommand{\PHFrontierMaxTokens}{733}
\providecommand{\PHFrontierMinTokens}{279}
\providecommand{\PHFrontierMatchedTokens}{510}
\providecommand{\PHFrontierLiPQuality}{99.3}
\providecommand{\PHFrontierMatchedPURPLEQuality}{98.5}
\providecommand{\PHFrontierReferenceQuality}{98.7}
\providecommand{\PHFrontierLiPTokensAtReference}{405}
\providecommand{\PHFrontierFixedTokens}{907}
\providecommand{\PHContinueOneBenefit}{18.0\%}
\providecommand{\PHContinueOneRedundant}{36.2\%}
\providecommand{\PHContinueOneNotWorth}{28.9\%}
\providecommand{\PHContinueOneHarm}{16.9\%}
\providecommand{\PHContinueThreeBenefit}{9.9\%}
\providecommand{\PHContinueThreeHarm}{35.9\%}
\providecommand{\PHOracleContinueOneBenefit}{28.2\%}
\providecommand{\PHStopRegret}{0.076}
\providecommand{\PHPrematureStop}{8.8\%}
\providecommand{\PHLateStop}{10.7\%}
\providecommand{\PHForcedKMax}{3.0\%}
\providecommand{\PHStopBrier}{0.116}
\providecommand{\PHStopECE}{0.028}
\providecommand{\PHStopAUROC}{0.839}
\providecommand{\PHStopAUPRC}{0.794}
\providecommand{\PHStopCalibrationPairs}{0.10/0.09, 0.30/0.26, 0.50/0.49, 0.70/0.73, 0.90/0.88}
\providecommand{\PHSpecificityTargetQuality}{100.0}
\providecommand{\PHSpecificityMatchedQuality}{99.1}
\providecommand{\PHSpecificityTargetLogLik}{0.041}
\providecommand{\PHSpecificityMatchedLogLik}{0.024}
\providecommand{\PHMatchCoverage}{87.5\%}
\providecommand{\PHMatchMeanSMD}{0.040}
\providecommand{\PHHumanFaithfulnessWinRange}{53.7--56.0\%}
\providecommand{\PHHumanStyleWinRange}{60.3--61.3\%}
\providecommand{\PHHumanOverallWinRange}{58.7\%}
\providecommand{\PHHumanKappaRange}{0.41--0.47}
\providecommand{\PHAblationFullQuality}{99.9}
\providecommand{\PHAblationMyopicQuality}{99.1}
\providecommand{\PHAblationLambdaZeroTokenIncrease}{44.8\%}
\providecommand{\PHAblationFullHarm}{15.0\%}
\providecommand{\PHAblationNoHarmGateHarm}{19.8\%}
\providecommand{\PHOracleLiPRegret}{0.069}
\providecommand{\PHOracleLiPSufficiency}{81.8\%}
\providecommand{\PHOracleLiPExcessTokens}{89}
\providecommand{\PHOracleLiPStopAgreement}{83.7\%}
\providecommand{\PHRobustDuplicateLengthDelta}{0.07}
\providecommand{\PHRobustUsefulHarm}{12.0\%}
\providecommand{\PHRobustConflictHarm}{19.0\%}
\providecommand{\PHRobustConflictQualityDrop}{0.7}
\providecommand{\PHOnlineLiPLatency}{370/440}
\providecommand{\PHOnlineLiPControllerLatency}{11}
\providecommand{\PHOfflineTeacherCalls}{9.7}
\providecommand{\PHOfflineCacheHit}{74.8\%}
\providecommand{\PHOfflineGPUHours}{336}
\providecommand{\PHOfflineStorageGB}{95}
\providecommand{\PHOfflineExtraGPUHours}{95}
\providecommand{\PHMainULaMPThreePurpleMAE}{0.638}
\providecommand{\PHMainULaMPThreeLiPMAE}{0.628}
\providecommand{\PHMainULaMPThreeMAEReduction}{1.5\%}
\providecommand{\PHOracleKStarShares}{17.7\%, 20.7\%, 18.0\%, 15.0\%, 11.0\%, 10.3\%, and 7.3\%}

\begin{table*}[!t]
    \vspace{-1.5em} 
    \centering
    \begingroup
    \small
    \setlength{\tabcolsep}{4.9pt}
    \renewcommand{\arraystretch}{1.1}
    \begin{tabular}{@{}ccccccccccc@{}}
    \toprule
    \rule{0pt}{1.2em}\textbf{Task} & Metric
    & BM25 & Contriever & IC-RALM & REPLUG
    & RankGPT & ICR & RSPG-Pre & PURPLE & \method{} \\
    \midrule
    \multicolumn{11}{@{}l}{\textit{With Qwen3.5-9B}} \\
    \textbf{LaMP-1} & Acc./F1
    & 63.1/61.5 & 64.5/63.0 & \underline{67.1}/\underline{65.6} & 65.7/64.1 & 67.0/65.3 & 64.8/63.4 & 66.2/64.7 & 66.7/65.1 & \textbf{68.6}$^{*}$/\textbf{67.0}$^{*}$ \\
    \textbf{LaMP-2} & Acc./F1
    & 50.2/45.8 & 49.8/44.9 & 50.1/46.0 & 48.7/44.7 & 50.2/46.0 & 51.2/46.6 & 50.3/46.1 & \underline{52.1}/\underline{47.0} & \textbf{53.9}$^{*}$/\textbf{48.6}$^{*}$ \\
    \textbf{LaMP-3} & MAE/RMSE
    & 31.7/68.8 & 31.4/67.6 & 29.3/66.2 & 28.5/62.4 & \underline{28.2}/\underline{61.9} & 30.5/67.0 & 29.4/66.4 & 29.0/65.5 & \textbf{26.9}$^{*}$/\textbf{58.3}$^{*}$ \\
    \textbf{LaMP-4} & R-1/R-L
    & 19.0/16.7 & 18.8/17.0 & 17.7/16.8 & 19.1/17.5 & 18.3/17.6 & 19.1/17.3 & 18.2/17.2 & \underline{19.4}/\underline{17.8} & \textbf{19.9}$^{*}$/\textbf{18.1}$^{*}$ \\
    \textbf{LaMP-5} & R-1/R-L
    & 45.5/39.9 & 46.1/40.7 & 45.8/42.3 & 45.7/41.8 & \underline{47.5}/\underline{42.8} & 46.4/41.8 & 46.1/41.3 & 45.3/42.5 & \textbf{48.2}$^{*}$/\textbf{43.6}$^{*}$ \\
    \textbf{LaMP-7} & R-1/R-L
    & 49.3/44.8 & 48.8/45.6 & 51.7/47.0 & 51.0/47.5 & 49.8/46.9 & 49.8/47.1 & 49.4/46.1 & \underline{51.9}/\underline{47.8} & \textbf{53.0}$^{*}$/\textbf{48.7}$^{*}$ \\
    \midrule
    \multicolumn{11}{@{}l}{\textit{With Llama-3.1-8B-Instruct}} \\
    \textbf{LaMP-1} & Acc./F1
    & 62.2/59.7 & 63.7/61.4 & 64.0/62.5 & 64.8/62.1 & \underline{66.1}/\underline{63.9} & 63.9/61.1 & 65.3/62.7 & 65.8/63.4 & \textbf{67.6}$^{*}$/\textbf{64.9}$^{*}$ \\
    \textbf{LaMP-2} & Acc./F1
    & 47.9/43.8 & 48.9/44.0 & 49.2/45.1 & 49.4/44.9 & 49.3/45.0 & 50.3/45.7 & 49.5/45.2 & \underline{51.2}/\underline{46.1} & \textbf{52.9}$^{*}$/\textbf{47.7}$^{*}$ \\
    \textbf{LaMP-3} & MAE/RMSE
    & 32.6/70.0 & 32.3/67.7 & 30.3/68.6 & 29.4/67.5 & 30.2/65.1 & 31.4/64.1 & \underline{29.1}/\underline{62.3} & 29.9/68.0 & \textbf{27.8}$^{*}$/\textbf{59.4}$^{*}$ \\
    \textbf{LaMP-4} & R-1/R-L
    & 17.9/17.1 & 17.1/16.6 & 17.3/16.4 & 18.6/17.0 & 18.5/16.3 & 18.6/16.9 & 17.8/16.8 & \underline{18.9}/\underline{17.3} & \textbf{19.4}$^{*}$/\textbf{17.6}$^{*}$ \\
    \textbf{LaMP-5} & R-1/R-L
    & 44.7/39.2 & 45.3/39.9 & 45.0/41.5 & 44.9/41.0 & 45.6/41.5 & \underline{46.7}/\underline{42.0} & 45.3/40.5 & 44.5/41.7 & \textbf{47.4}$^{*}$/\textbf{42.8}$^{*}$ \\
    \textbf{LaMP-7} & R-1/R-L
    & 47.9/43.5 & 47.4/44.3 & 50.3/45.7 & 49.6/46.2 & 48.4/45.6 & 48.4/45.8 & 48.0/44.8 & \underline{50.5}/\underline{46.5} & \textbf{51.6}$^{*}$/\textbf{47.4}$^{*}$ \\
    \bottomrule
    \end{tabular}
    \endgroup
    \caption{Main results on the LaMP time-based split with Qwen3.5-9B (top) and Llama-3.1-8B-Instruct (bottom) as frozen task generators. Accuracy, F1, and ROUGE are reported on a 0--100 scale; MAE and RMSE are multiplied by 100. 
    The best and second-best results within each generator are shown in bold and underlined, respectively.
    The superscript $^{*}$ denotes a statistically significant improvement over the strongest baseline under a two-sided paired $t$-test across the same ten runs ($p<0.05$).
    }
    \label{tab:main-results}
    \vspace{-1.0em}
    \end{table*}

\section{Experiments}
\label{sec:experiments}

\subsection{Experimental Setup}
\label{sec:experimental-setup}

\paragraph{Dataset and Evaluation.}
Our experiments use the \textbf{LaMP} benchmark \citep{salemi-etal-2024-lamp}, which covers six personalized tasks, including three classification tasks—\textbf{LaMP-1} (Citation Identification, binary), \textbf{LaMP-2} (Movie Tagging, multiclass with 15 categories), and \textbf{LaMP-3} (Product Rating, ordinal from 1–5 stars)and four text generation tasks: \textbf{LaMP-4} (News Headline Generation), \textbf{LaMP-5} (Scholarly Title Generation), and \textbf{LaMP-7} (Tweet Paraphrasing).

Following previous work (Salemi et al., 2024b,a; Shi et al., 2025), we report accuracy(Acc.) and F1 score for LaMP-1 and LaMP-2, mean absolute error (MAE) and root mean squared error (RMSE) for LaMP-3. We evaluate text generation performance on LaMP-4, LaMP-5, and LaMP-7 using ROUGE-1 (R-1) and ROUGE-L (R-L) (Lin, 2004).
Dataset statistics, the time-based (T) split protocols, and detailed metric definitions are reported in Appendix~\ref{app:task-setup}.

\paragraph{Implementation Details.}

We use Qwen3.5-9B\citep{qwen-team-2026-qwen35} as the primary frozen task generator in all main experiments. We repeat the evaluation with Llama-3.1-8B-Instruct\citep{dubey-etal-2024-llama3} as an alternative frozen task generator. Our frozen retriever is Contriever \citep{izacard-etal-2022-contriever}, which returns at most $M=20$ legal records. Further implementation details are provided in Appendix~\ref{app:implementation}.

\subsubsection{Baselines}
\label{sec:baselines}

The main comparison contains eight retrieval-augmented baselines. BM25 \citep{robertson-zaragoza-2009-bm25} and Contriever \citep{izacard-etal-2022-contriever} provide sparse and frozen dense relevance rankings. RankGPT \citep{sun-etal-2023-rankgpt} performs zero-shot listwise permutation reranking, whereas ICR \citepalias{chen-etal-2025-icr} derives calibrated relevance scores from query-to-record attention. IC-RALM \citep{ram-etal-2023-icralm} incorporates records through decoding-time context switching, while REPLUG \citep{shi-etal-2024-replug} marginalizes token predictions from record-conditioned generation streams using normalized retrieval scores.
RSPG-Pre \citep{salemi-etal-2024-optimization} uses a Longformer selector before generation to choose among multiple retrievers.
PURPLE \citep{du-etal-2026-optimizing} learns an order-sensitive Plackett--Luce profile policy using the frozen generator's reference-response log-likelihood as reward.
Detailed baseline configurations, shared budgets and prompts are provided in Appendix~\ref{app:implementation}.

\subsection{Main Results}
\label{sec:main-results}

\method{} consistently achieves the best performance across tasks, metrics, and generator backbones. As shown in Table~\ref{tab:main-results}, it ranks first in all 24 reported metric cells: six tasks with two metrics each under both Qwen3.5-9B and Llama-3.1-8B-Instruct. 
With Qwen3.5-9B, \method{} improves over the strongest baseline by 1.5/1.4 points on LaMP-1 and by 1.8/1.6 points on LaMP-2. On the ordinal-rating task LaMP-3, it reduces MAE and RMSE by 1.3 and 3.6 points, respectively. The gains extend to all generation tasks: \method{} improves ROUGE-1/ROUGE-L by 0.5/0.3 points on LaMP-4, 0.7/0.8 points on LaMP-5, and 1.1/0.9 points on LaMP-7. Together, these results establish the effectiveness of \method{} across classification, ordinal prediction, and personalized generation. The same pattern transfers to Llama-3.1-8B-Instruct. The consistent improvements across two frozen generators support the robustness of \method{} to the choice of task backbone.

\subsection{Ablation Analysis}
\label{sec:ablations}

Figure~\ref{fig:ablations} compares four single-component ablations. \emph{w/o Long-horizon} replaces long-horizon supervision with myopic targets, \emph{w/o Adaptive \stopa{}} forces the development-selected $k$, \emph{w/o User-Specificity} sets $\beta=0$, and \emph{w/o Token-Cost} sets $\lambda=0$.

\paragraph{Long-horizon supervision and adaptive stopping.}
Both components are important for selecting sufficient rather than merely relevant history. Replacing long-horizon supervision with myopic targets significantly degrades task quality, lengthens profiles, and increases harm, showing that locally attractive records do not reliably form a useful final profile.
Disabling adaptive stopping produces the same pattern: fixed-length selection cannot avoid unnecessary evidence.
Together, these results validate reasoning about future selections when additional personalization is no longer beneficial.

\paragraph{User specificity and token cost play distinct roles.}
Removing user specificity significantly reduces task quality though profiles become slightly shorter. Removing token cost yields a slight improvement in aggregate task quality, but significantly increases both token usage while making early stopping much less frequent. The token-cost term is therefore useful: it prevents marginal quality gains from being purchased with disproportionately long and riskier profiles, yielding a better quality--cost--harm trade-off.

\begin{figure}[!t]
\centering
\includegraphics[width=\columnwidth]{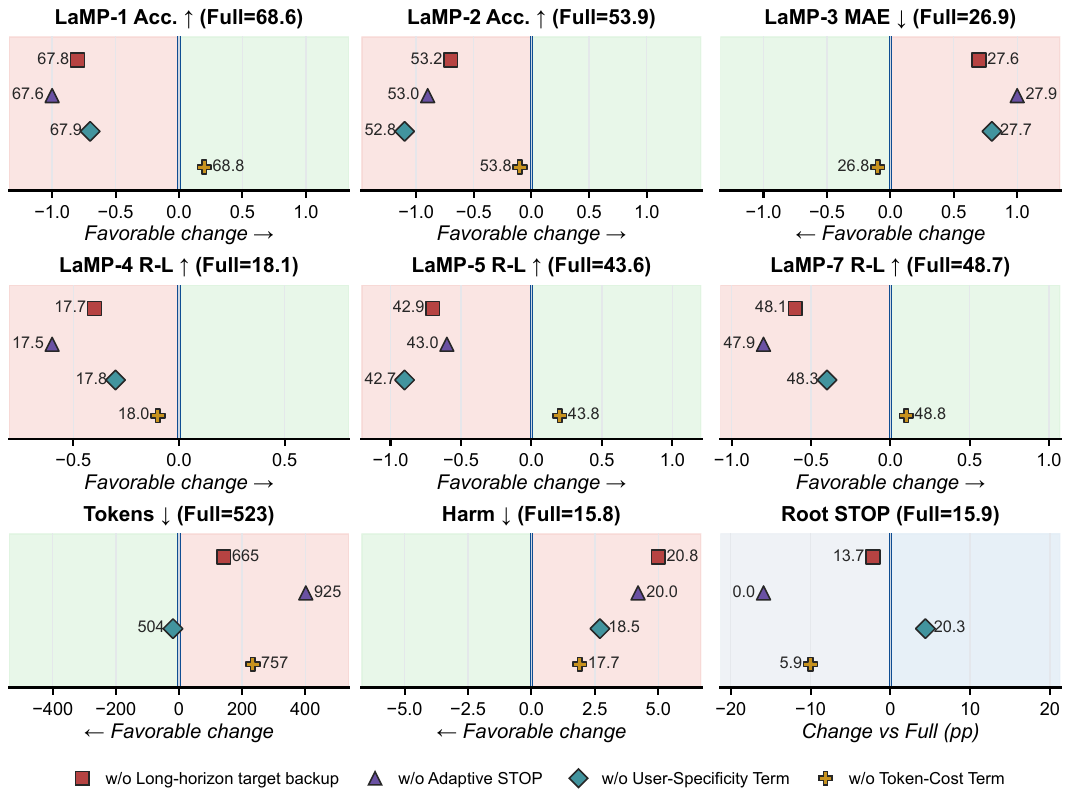}
\caption{Component ablations. Marker positions show mean changes from full \method{}. Rightward is favorable for Acc. and R-L, whereas leftward is favorable for MAE, Tokens, and Harm. Root \stopa{} has no single favorable direction.}
\label{fig:ablations}
\vspace{-1.5em}
\end{figure}

\subsection{Is More Personalization Always Beneficial?}
\label{sec:always-on}

\paragraph{Longer fixed profiles are not always better.}
Figure~\ref{fig:always-beneficial}(a) shows that fixed-length personalization does not improve monotonically as more history is added. Here, quality aggregates direction-normalized evaluation metrics across all six tasks. Among the evaluated lengths, quality is highest at $k=\PHFixedPeakK{}$ (\PHFixedPeakQuality{}) but falls by \PHFixedPeakToTenQualityGap{} points at $k=10$, while harm relative to the empty profile rises from \PHFixedFiveHarm{} to \PHFixedTenHarm{}. This trade-off motivates instance-adaptive rather than fixed profile lengths. Appendix Table~\ref{tab:app-fixed-k} reports the measured quality, token cost, and harm values for the full fixed-length grid.

\paragraph{Overriding predicted \stopa{} becomes increasingly risky.}
Figure~\ref{fig:always-beneficial}(b) forces the controller to append its next one, two, or three records after predicted \stopa{}. From $+1$ to $+3$, the beneficial share falls from \PHContinueOneBenefit{} to \PHContinueThreeBenefit{}, whereas the harmful share rises from \PHContinueOneHarm{} to \PHContinueThreeHarm{} alongside increasing token cost. Thus, deeper continuation is less likely to improve the objective and more likely to hurt performance. Appendix Table~\ref{tab:app-forced-continuation} gives the policy-by-depth breakdown, and Appendix Figure~\ref{fig:app-stop-calibration} reports calibration of the learned stop confidence to the bounded-tree teacher.

\begin{figure}[!t]
    \centering
    \includegraphics[width=\linewidth]{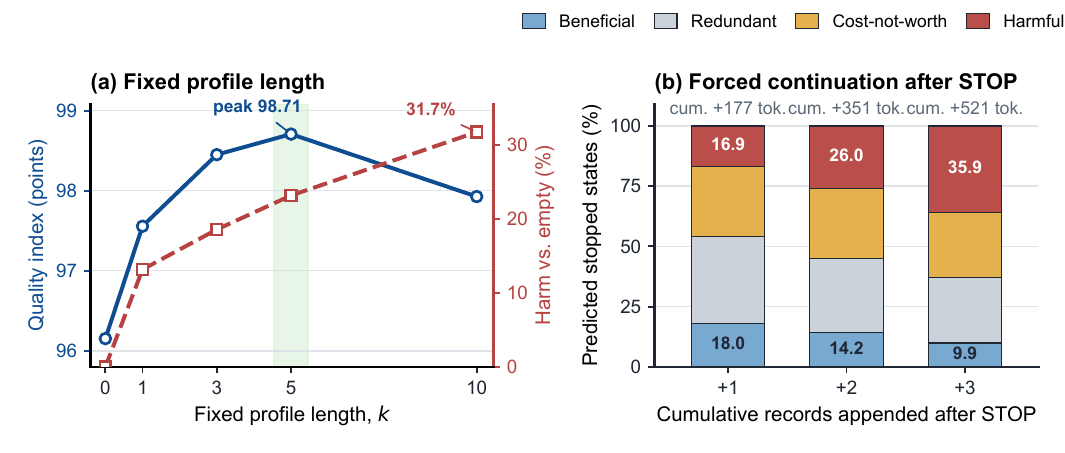}
    \caption{Measured profile-length and stopping diagnostics.
    (a) Direction-normalized quality and harm rate relative to the empty profile for fixed-$k$ Contriever profiles.
    (b) Outcome shares after appending one to three controller-selected records beyond predicted \stopa{} prefixes.
    }
\label{fig:always-beneficial}
\end{figure}

\subsection{User Specificity}
\label{sec:specificity-results}

\begin{table}[!t]
\centering
\small
\setlength{\tabcolsep}{2.0pt}
\begin{tabular}{lcccc}
\toprule
Intervention & Quality & Ref.-LL gain & $\Delta$ target & Match SMD \\
\midrule
Target-user profile & 100.0 & +0.041 & +0.001 & 0.002 \\
Matched other user & 99.1 & +0.024 & -0.018 & 0.041 \\
Action shuffle & 98.7 & +0.018 & -0.025 & 0.039 \\
Owner permutation & 98.6 & +0.011 & -0.029 & 0.044 \\
Random other user & 98.5 & +0.012 & -0.031 & 0.180 \\
Whole-profile swap & 98.0 & +0.005 & -0.037 & 0.191 \\
\bottomrule
\end{tabular}
\caption{Frozen-generator specificity interventions. Quality is normalized so the target-user profile is 100; likelihood values are per target token.}
\label{tab:specificity-interventions}
\end{table}

Content-matched other-user replacement shifts the normalized quality index from \PHSpecificityTargetQuality{} to \PHSpecificityMatchedQuality{} and reference-likelihood gain from \PHSpecificityTargetLogLik{} to \PHSpecificityMatchedLogLik{} per token (Table~\ref{tab:specificity-interventions}). This matched audit covers \PHMatchCoverage{} of records (mean absolute standardized difference: \PHMatchMeanSMD{}), excluding unmatched cases. Together with the $\beta=0$, context-only, action-only, and no-harm-gate rows, these interventions assess whether the observed gains arise from target-user-specific information rather than generic contextual utility.

\paragraph{Human evaluation.}
In the human evaluation on LaMP-4 and LaMP-7, the observed win-rate ranges are \PHHumanFaithfulnessWinRange{} for task faithfulness, \PHHumanStyleWinRange{} for user-style consistency, and \PHHumanOverallWinRange{} overall, with Fleiss' $\kappa$ spanning \PHHumanKappaRange{} (Appendix~\ref{app:specificity-audit}). The study randomized presentation order, concealed method identity, and adjudicated only protocol violations; ties were retained as valid outcomes.

\subsection{Hyperparameter Sensitivity}
\label{sec:hyperparameter-sensitivity}

We examine how the two objective weights shape distinct aspects of profile construction. We vary $\beta$ and $\lambda$ one at a time; at each setting, we hold the remaining configuration fixed, recompute tree backups, and train the corresponding controller. Figure~\ref{fig:beta-lambda-sensitivity} pairs the direction-normalized quality index across all six tasks (higher is better) with specificity gain for $\beta$ and profile length for $\lambda$.

\begin{figure}[!t]
    \centering
    \includegraphics[width=\columnwidth]{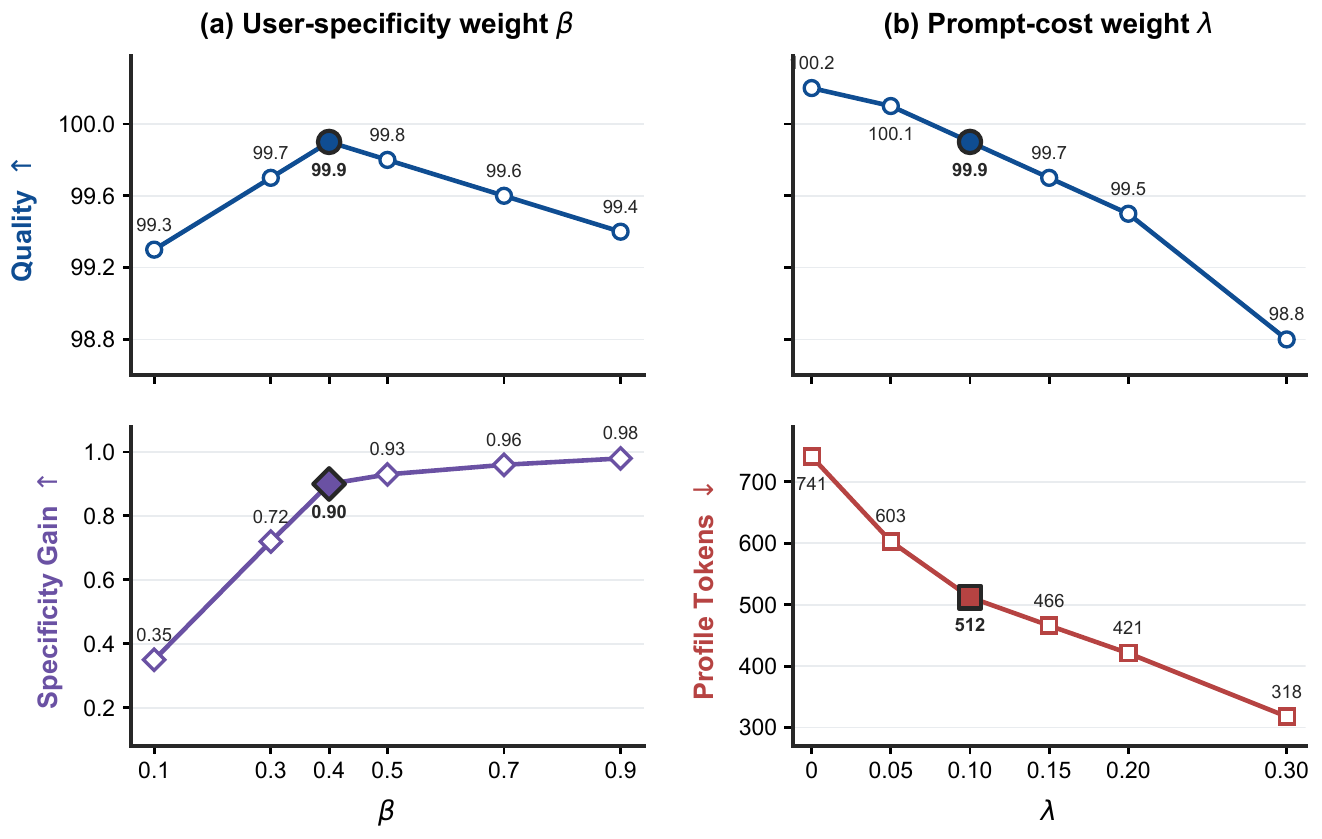}
    \caption{
    Hyperparameter sensitivity of ENOUGH to $\beta$ (specificity weight, left) and $\lambda$ (token-cost weight, right).}
    \label{fig:beta-lambda-sensitivity}
    \end{figure}

\begin{figure}[!t]
    \centering
    \includegraphics[width=\linewidth]{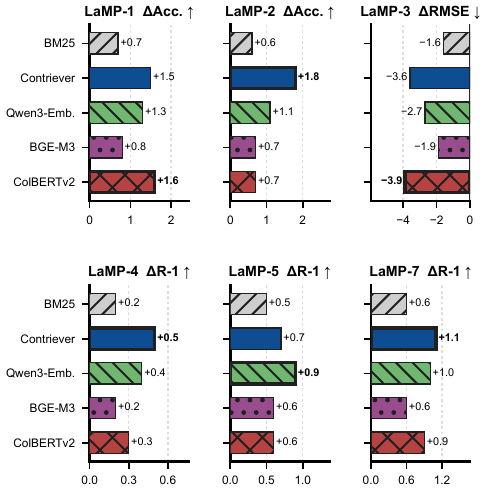}
    \caption{Performance gaps between ENOUGH and the strongest compatible baseline across different retrievers.}
    \label{fig:retriever-delta-taskwise}
    \end{figure}

The sweeps motivate the selected operating point $(\beta,\lambda)=(0.4,0.10)$. In the $\beta$ sweep, $\beta=0.4$ gives the highest quality index (99.9) while reaching a specificity gain of 0.90, close to the observed maximum of 0.98. In the $\lambda$ sweep, $\lambda=0.10$ reduces mean profile length from 741 to 512 tokens (30.9\%) relative to $\lambda=0$, while the quality index changes from 100.2 to 99.9. Additional development-set sweeps over $M$, $K_{\max}$, $K^{-}$, $(J,W)$, and the training-loss weights $\eta$ and $\zeta$ are reported in Appendix~\ref{app:additional-hyperparameter-sensitivity}.

\subsection{Retriever Independence}
\label{sec:retriever-independence}

We test whether the gains of \method{} are specific to Contriever by repeating the evaluation with BM25, Contriever, Qwen3-Embedding-0.6B~\citep{zhang-etal-2025-qwen3-embedding}, BGE-M3~\citep{chen-etal-2024-m3}, and ColBERTv2~\citep{santhanam-etal-2022-colbertv2}. 
For each retriever, \method{} and all compatible baselines use the same retriever-specific candidate pool and generation settings. Figure~\ref{fig:retriever-delta-taskwise} reports the resulting test-set gap, $\Delta=\text{\method{}}-\text{strongest baseline}$.
\method{} outperforms the strongest baseline in all 30 retriever--task combinations. Across retrievers, the Accuracy and ROUGE-1 gains range from $+0.2$ to $+1.8$ points. The largest margin occurs under different retrievers across tasks, suggesting that the benefit of \method{} comes from its selection and stopping strategy rather than from a particular retrieval model.

\begin{table}[!t]
\centering
\small
\begin{tabular}{lrrrrr}
\toprule
Method & Quality & Tokens & Calls & p50 ms & p95 ms \\
\midrule
\method{} & 100.00 & 510 & 1 & 370 & 440 \\
PURPLE & 96.49 & 905 & 1 & 398 & 474 \\
RankGPT & 96.32 & 869 & 6 & 1120 & 1379 \\
REPLUG & 95.38 & 1388 & 5 & 1323 & 1595 \\
IC-RALM & 94.82 & 1210 & 7 & 1463 & 1793 \\
ICR & 94.30 & 819 & 4 & 865 & 1049 \\
RSPG-Pre & 94.29 & 765 & 3 & 796 & 951 \\
Contriever & 92.17 & 931 & 1 & 365 & 442 \\
BM25 & 90.96 & 905 & 1 & 360 & 421 \\
\bottomrule
\end{tabular}
\caption{Measured online inference cost under identical generation settings. Tokens is the mean selected-profile length; Calls is the number of generator-equivalent invocations; p50/p95 are end-to-end latency in milliseconds.}
\label{\efficiencytablelabel}
\end{table}

\subsection{Efficiency}
\label{sec:efficiency}

Table~\ref{tab:efficiency} shows that \method{} achieves the highest six-task normalized Quality (100.00) with the shortest selected profile (510 tokens) and one generator-equivalent call. Compared with the strongest baseline, PURPLE, it improves Quality by 3.51 points while using 43.6\% fewer profile tokens and reducing p50/p95 latency from 398/474 to 370/440 ms. BM25 and Contriever have comparable one-call latency but trail \method{} by 9.04 and 7.83 Quality points, respectively. Multi-call baselines require 3--7 calls and exhibit substantially higher p50/p95 latency (796--1463/951--1793 ms), indicating that the Quality gain of \method{} does not rely on additional online calls. Appendix~\ref{app:efficiency} separate the offline training-time and task-family latency views.

\section{Conclusion}
\label{sec:conclusion}

This paper introduced \method{}, a framework for minimal sufficient personalization that jointly decides whether to personalize, which behavioral records to include, and when to stop. By distilling offline counterfactual evaluations of downstream gain, user specificity, and token cost into a long-horizon controller, \method{} constructs input-adaptive profiles while requiring only one frozen-generator call at inference. Across six LaMP tasks and two generator backbones, \method{} achieves the best result in all 24 task--metric settings; it also improves normalized quality by 3.51 points over the strongest baseline while using 43.6\% fewer profile tokens. These results show that effective personalization depends on identifying when selected user evidence is sufficient, not on exposing models to more history. Future work will extend this formulation to evolving histories and interactive settings.


\bibliography{references}

\clearpage
\appendix
\renewcommand{\efficiencytablelabel}{tab:app-efficiency}
\section{Related Work}
\label{app:related-work}

\subsection{Personalized Large Language Models}
Personalized large language models adapt generation to user histories,
preferences, or interaction patterns through in-context records, explicit
profiles, or user-specific parameters
\citep{zhang-etal-2025-personalization-survey,xu-etal-2025-personalized}.
Guided Profile Generation distills sparse personal context into concise
natural-language descriptions \citep{zhang-2024-guided}, whereas OPPU assigns
each user a parameter-efficient module and combines this parametric memory with
retrieved profile information \citep{tan-etal-2024-democratizing}. These
approaches provide history abstraction or persistent user modeling. In
contrast, our setting freezes both retriever and generator, represents history
as behavioral context--action records, and shares a task-conditioned controller
across users and tasks. \method{} constructs a transparent ordered profile for
each input, maintains no user-specific parameters, and may select no record
when personalization is not worth its prompt-token cost.

Reliable personalization must also separate user-specific benefit from generic
task improvement. ExPerT evaluates content and writing-style alignment
\citep{salemi-etal-2025-expert}, while CFRAG uses other users' histories as
generation evidence \citep{shi-etal-2025-cfrag}. Because a generically useful
in-context example can improve reference likelihood even when its owner is
irrelevant, \method{} combines gain over the empty profile with an offline
matched-replacement comparison. For each selected slot, it contrasts the
target-user record with other-user records matched on task, pre-action content,
query relevance, length, and time when available, while holding the remaining
profile fixed. Specificity is credited only if the complete target-user profile
also improves over no personalization. Matched records supervise the controller
but are discarded at inference; the resulting score estimates a replacement
effect for the frozen generator, not a causal effect of user preference.

\subsection{Retrieval-Augmented Personalized Large Language Models}
Retrieval-augmented personalization conditions a language model on selected
user-history records without updating the generator for each user. LaMP
established this setting across classification and generation tasks
\citep{salemi-etal-2024-lamp}. Later work optimizes selection for downstream
utility: RSPG chooses among no retrieval and multiple personalized retrievers
using generation feedback \citep{salemi-etal-2024-optimization}; CFRAG expands
the evidence pool with similar users and learns retrieval and reranking
\citep{shi-etal-2025-cfrag}; and PURPLE trains an order-sensitive
Plackett--Luce profile policy from reference-response likelihood
\citep{du-etal-2026-optimizing}. General-purpose rerankers RankGPT
\citep{sun-etal-2023-rankgpt} and ICR \citepalias{chen-etal-2025-icr}
provide strong pool-reordering baselines but do not model profile termination.
CFRAG returns a top-$k$ set, PURPLE an ordered $K$-permutation, and our LaMP
adaptations of RankGPT and ICR a preset number of records. RSPG may choose no
retrieval but does not make record-wise continuation decisions after selecting
a retriever. \method{} instead treats record inclusion and \stopa{} as actions
in one sequential problem after frozen candidate retrieval. It can return an
empty profile or any feasible length up to $K_{\max}$, evaluating terminal
profiles by downstream gain, incremental prompt-token cost, and user-specific
benefit.

Other retrieval-augmented systems alter how evidence enters generation:
IC-RALM switches context during decoding \citep{ram-etal-2023-icralm}, whereas
REPLUG marginalizes predictions from multiple record-conditioned streams
\citep{shi-etal-2024-replug}. These methods primarily address
knowledge-centric evidence and use inference paths different from personalized
profile construction. \method{} makes all inclusion and stopping decisions
before generation with a lightweight controller trained from bounded-tree
outcomes, then invokes the frozen generator exactly once; a root \stopa{} means
no profile is injected, although candidate retrieval has already occurred.

\clearpage

\section{Notation}
\label{app:notation}

\begin{table}[H]
\centering
\small
\begin{tabular}{ll}
\toprule
Symbol & Meaning \\
\midrule
$C_i$ & top-$M$ candidate records from the frozen retriever \\
$\mathbf S_t$ & ordered profile after $t$ inclusion actions \\
$g_i(\mathbf S)$ & downstream likelihood gain over the empty profile \\
$p_i(\mathbf S)$ & matched-replacement user specificity \\
$c_i(\mathbf S)$ & normalized incremental prompt-token cost \\
$U_i(\mathbf S)$ & cost-free personalized utility \\
$J_i^{\lambda}(\mathbf S)$ & cost-regularized terminal objective \\
$Q_{\theta}^{g,p,c}$ & predicted outcome components of one action \\
$\stopa{}$ & terminate with the current profile \\
\bottomrule
\end{tabular}
\caption{Core notation.}
\label{tab:notation}
\end{table}

\FloatBarrier

\section{Algorithms}
\label{app:algorithms}

\begin{algorithm}[H]
\caption{Offline bounded-tree supervision}
\label{alg:offline}
\begin{algorithmic}[1]
\Require Training data $\mathcal D$; frozen retriever $R$; frozen generator $G$;
candidate limit $M$; horizon $K_{\max}$; search budgets $J,W$
\Ensure Labeled state-action set $\mathcal Z$
\State $\mathcal Z\gets\varnothing$
\For{each training instance $i$}
  \State $C_i\gets R(x_i,H_{u_i};M)$ and remove target duplicates
  \State Precompute fixed matched replacements for every $d\in C_i$
  \State Sample partial-profile roots from the current state-proposal mixture
  \For{each sampled root state $s$}
    \State Build bounded prefix tree $\mathcal T_i(s)$ with budgets $J,W$
    \State Score every unique terminal prefix and its sampled slot replacements
    \State Back up $J_i^{\lambda}$ from leaves to expanded actions
    \For{each labeled $(s',a)$ in $\mathcal T_i(s)$}
      \State Propagate the three components of the same best leaf
      \State $\mathcal Z\gets\mathcal Z\cup\{(s',a,\mathbf R_i(s',a))\}$
    \EndFor
  \EndFor
\EndFor
\State \Return $\mathcal Z$
\end{algorithmic}
\end{algorithm}

\paragraph{Offline complexity.}
With $N_s$ sampled roots, horizon $K_{\max}$, search width $W$, action budget $J$, and $K^{-}$ matched replacements, offline label construction requires at most
\begin{equation}
    O\!\left(N_sK_{\max}WJ(K^{-}+2)\right)
\end{equation}
teacher-forced scoring operations before prefix caching. The $K^{-}+2$ factor accounts for the target profile, the leave-one-out profile, and the matched replacements for the sampled slot. These claims assume cached candidate representations and do not include the frozen candidate-retrieval cost.

\begin{algorithm}[H]
  \caption{Minimal sufficient profile construction}
  \label{alg:inference}
  \begin{algorithmic}[1]
  \Require Input $(\tau,x,H_u)$; frozen retriever $R$; controller $Q_{\theta}$;
  frozen generator $G$
  \Ensure Ordered profile $\widehat{\mathbf S}$ and output $\widehat y$
  \State $C\gets R(x,H_u;M)$; $\mathbf S\gets\varnothing$
  \While{$|\mathbf S|<\min(K_{\max},|C|)$}
    \State Mask repeated records and actions exceeding $L_{\max}$
    \State Score all feasible record actions and \stopa{} in one controller batch
    \State $a\gets\arg\max_{a'\in\mathcal A(s)}Q_{\theta}^{\mathrm{net}}(s,a')$
    \If{$a=\stopa{}$}
      \State \textbf{break}
    \EndIf
    \State $\mathbf S\gets\mathbf S\oplus a$
  \EndWhile
  \State $\widehat y\gets G(P_{\tau}(x,\mathbf S))$
  \State \Return $(\mathbf S,\widehat y)$
  \end{algorithmic}
  \end{algorithm}

\paragraph{Test-time procedure.}
Algorithm~\ref{alg:inference} gives the inference procedure. Candidate record embeddings are cached and all remaining actions are scored in a batch at each step. Neither record selection nor the \stopa{} decision uses a reference output, matched replacement, generator likelihood, counterfactual rollout, or generator call. The controller performs at most $K_{\max}$ lightweight decision steps. Once profile construction terminates, the autoregressive generator is called exactly once on the final prompt.

\paragraph{Online complexity.}
Online profile construction uses at most $O(K_{\max}M)$ batched lightweight action scores and exactly one final generator call. This bound assumes cached candidate representations and excludes the frozen candidate-retrieval cost.

\FloatBarrier

\section{Proofs}
\label{app:proofs}

\begin{table*}[t]
\centering
\small
\setlength{\tabcolsep}{3.0pt}
\begin{tabular}{lccccccc}
\toprule
Task & \#Train & \#Dev & \#Test & Input length & Output length & Profile size & \#Labels \\
\midrule
\shortstack[l]{LaMP-1: Personalized Citation Identification} & 6,542 & 1,500 & 1,500 & $51.44\mathbin{\pm}5.71$ & -- & $84.16\mathbin{\pm}47.54$ & 2 \\
\midrule
\shortstack[l]{LaMP-2: Personalized Movie Tagging} & 5,073 & 1,410 & 1,557 & $92.40\mathbin{\pm}21.95$ & -- & $86.76\mathbin{\pm}189.53$ & 15 \\
\midrule
\shortstack[l]{LaMP-3: Personalized Product\ Rating} & 20,000 & 2,500 & 2,500 & $128.18\mathbin{\pm}146.25$ & -- & $185.41\mathbin{\pm}129.30$ & 5 \\
\midrule
\shortstack[l]{LaMP-4: Personalized News Headline Generation} & 12,500 & 1,500 & 1,800 & $29.98\mathbin{\pm}12.09$ & $10.08\mathbin{\pm}3.11$ & $204.60\mathbin{\pm}250.75$ & -- \\
\midrule
\shortstack[l]{LaMP-5: Personalized Scholarly Title Generation} & 14,682 & 1,500 & 1,500 & $162.34\mathbin{\pm}65.64$ & $9.72\mathbin{\pm}3.21$ & $87.89\mathbin{\pm}53.64$ & -- \\
\midrule
\shortstack[l]{LaMP-7: Personalized Tweet Paraphrasing} & 13,437 & 1,498 & 1,500 & $29.72\mathbin{\pm}7.01$ & $16.97\mathbin{\pm}5.68$ & $15.71\mathbin{\pm}14.86$ & -- \\
\bottomrule
\end{tabular}
\caption{Statistics of the LaMP time-based (T) split used in our main experiments. Counts are instances; input and output lengths are whitespace-delimited words; profile size is the number of historical records. Length and profile statistics are mean $\mathbin{\pm}$ standard deviation over all train, development, and test instances. Classification and rating outputs are labels, so their output lengths are omitted.}
\label{tab:app-dataset-statistics}
\end{table*}

\subsection{Proposition~\ref{prop:exact-value-stopping}: Exact-value stopping}
\label{app:proof-exact-value-stopping}

\begin{proof}
Fix an instance $i$ and a reachable state $s$ whose current ordered profile
prefix is $\mathbf S$. Let $\mathcal L_i(\mathbf S)$ be the set of terminal
profile prefixes reachable from $\mathbf S$, including $\mathbf S$ itself via
the \stopa{} action. Because $C_i$ is finite, records cannot repeat, and profile
length is at most $K_i$, every candidate subtree is finite. Let $h(\mathbf S)$
denote the maximum number of record-action edges on a path starting from
$\mathbf S$.

We first establish, by induction on $h(\mathbf S)$, that
\begin{equation}
  V_i^{\star}(\mathbf S)
  =\max_{\mathbf L\in\mathcal L_i(\mathbf S)}
    J_i^{\lambda}(\mathbf L).
  \label{eq:descendant-value}
\end{equation}
If $h(\mathbf S)=0$, no record action is feasible, so
$\mathcal L_i(\mathbf S)=\{\mathbf S\}$ and the boundary condition gives
$V_i^{\star}(\mathbf S)=J_i^{\lambda}(\mathbf S)$. Now suppose
$h(\mathbf S)>0$ and that Eq.~\eqref{eq:descendant-value} holds for every
child prefix. The terminal descendants decompose as
\begin{equation}
  \mathcal L_i(\mathbf S)
  =\{\mathbf S\}\cup
    \bigcup_{d\in\mathcal A(s)\setminus\{\stopa\}}
      \mathcal L_i(\mathbf S\oplus d).
\end{equation}
Applying the induction hypothesis to each child and then Eq.~\eqref{eq:bellman}
yields
\begin{align}
  \max_{\mathbf L\in\mathcal L_i(\mathbf S)}J_i^{\lambda}(\mathbf L)
  &=\max\left\{
      J_i^{\lambda}(\mathbf S),
      \max_{d\in\mathcal A(s)\setminus\{\stopa\}}
        V_i^{\star}(\mathbf S\oplus d)
    \right\}\\
  &=V_i^{\star}(\mathbf S),
\end{align}
which proves Eq.~\eqref{eq:descendant-value}.

Consider now a policy that selects an action maximizing $Q_i^{\star}(s,a)$.
By the definitions of the action values and the Bellman recursion,
\begin{equation}
  \max_{a\in\mathcal A(s)}Q_i^{\star}(s,a)
  =V_i^{\star}(\mathbf S).
\end{equation}
If the policy selects \stopa{}, then
$J_i^{\lambda}(\mathbf S)=Q_i^{\star}(s,\stopa)
=V_i^{\star}(\mathbf S)$, so Eq.~\eqref{eq:descendant-value} shows that the
current prefix is optimal among all terminal descendants. If it selects a
record $d$, then
$V_i^{\star}(\mathbf S\oplus d)=Q_i^{\star}(s,d)
=V_i^{\star}(\mathbf S)$. The child subtree has smaller height, so repeating
the same argument eventually reaches a terminal prefix $\mathbf L$ with
$J_i^{\lambda}(\mathbf L)=V_i^{\star}(\mathbf S)$. Termination is guaranteed
because every record action descends one level in a finite tree.

When \stopa{} is tied with a record action, it is itself a maximizing action;
the \stopa{} case above therefore shows that resolving the tie in its favor
preserves optimality. Applying the result at the root proves the proposition.
\end{proof}

This guarantee is conditional on access to the exact values $Q_i^{\star}$. It
does not imply that the learned controller $Q_{\theta}$ recovers the optimum of
the complete combinatorial candidate tree.

\section{Supplementary Experimental Details}
\label{app:experiments}

\subsection{Datasets, Time-Based Split Protocol, and Evaluation}
\label{app:task-setup}

We use the six publicly available tasks in the official LaMP snapshot: LaMP-1--5 and LaMP-7. LaMP-6 is excluded because its publicly distributed files contain only identifiers from the licensed Avocado email collection rather than the underlying email text. For the main experiments, we use the time-based (T) protocol, in which interactions are ordered chronologically: earlier records form the profile and later interactions form the prediction targets, thereby evaluating personalization for future interactions of existing users. Table~\ref{tab:app-dataset-statistics} reports time-based statistics computed directly from the files used in the main experiments.

\FloatBarrier

\subsection{Implementation Details}
\label{app:implementation}

\paragraph{Models and search.}
Our frozen retriever is Contriever, which returns at most $M=20$ legal records in the main setting. We use Qwen3.5-9B as the main LLM and Llama-3.1-8B-Instruct for generator transfer; both LLMs are frozen and use deterministic decoding with their respective official chat templates. Qwen3.5-9B is served through its text-only path with \texttt{enable\_thinking=false}. Teacher-forced likelihood is computed only over reference-response tokens, excluding prompt tokens and all chat-template, control, and thinking tokens. The controller has $K_{\max}=10$, uses $K^-=3$ matched alternatives per candidate, and is trained from bounded-tree targets with the policy-induced offline aggregation described below. The search uses $J=8$ candidate expansions and beam width $W=4$ at the default operating point. We set the ranking- and stopping-loss weights in Eq.~\eqref{eq:total-loss} to $\eta=0.2$ and $\zeta=0.4$, respectively. Frozen-model likelihoods and tokenized record encodings are cached across selectors and $\lambda$ values. For each generator, we independently recompute tokenizer costs, likelihoods, controller labels, and caches rather than reusing artifacts produced by another model.

\paragraph{Policy-induced offline aggregation.}
Algorithm~\ref{alg:offline} is the label-construction routine used for both rounds of offline supervision. We first apply it to the initial state-proposal mixture and optimize the controller under Eq.~\eqref{eq:total-loss}. After warm-start convergence, one policy-induced offline data-aggregation round adds states visited by the current controller, reuses the same frozen generator to label them, and retrains the controller on the union of both state pools. No controller update is performed during test-time profile construction.

\paragraph{Baseline configurations.}
BM25 and Contriever independently retrieve the top $K=5$ records from the same leakage-filtered user history using the official task-specific query fields. RankGPT, ICR, and PURPLE receive the same top $M=20$ pool from frozen Contriever and return $K=5$ records; following the published LaMP adaptation, RankGPT and ICR use frozen Llama-3.1-8B-Instruct as a separate reranker. IC-RALM and REPLUG use Contriever evidence under the same five-record budget but retain their native context-switching and marginalization mechanisms; for REPLUG, we use the LM-supervised retrieval variant from its LaMP adaptation. RSPG-Pre operates over its native retriever pool and is retrained with our task generator and $K=5$. \method{} and the single-prompt baselines invoke the task generator once after profile construction, whereas IC-RALM and REPLUG retain their native multi-pass inference. All methods share the record serialization, task prompt, frozen task generator, decoding configuration, and leakage controls. Empty/full profiles, fixed-$k$ and fixed-token controls, Recency, Random, MMR, ROPG-KD, the myopic controller, and forced-length \method{} variants are retained as diagnostics or ablations rather than as main baselines.

\paragraph{Serialization, leakage controls, and matching.}
Each profile record is serialized with a task tag and explicit field boundaries:
\begin{quote}\small
\texttt{[TASK] $\tau$ [HIST\_CONTEXT] $x_d$ [USER\_ACTION] $a_d$ [META] $m_d$}.
\end{quote}
The current input is serialized separately and never copied into the profile. We preserve complete record boundaries, mask an action if its complete serialization would overflow the context window, and never silently truncate records. Within each generator condition, all selectors use the same generator-specific tokenizer and prompt construction. The empty-profile prompt contains the same instructions and delimiters but no historical record, so its comparison changes only the augmentation content. Appendix~\ref{app:prompt-templates} shows the complete logical system and user messages for all six tasks.

For every target, the legal user history contains only interactions whose timestamps strictly precede the target timestamp. We apply this temporal truncation before constructing the retrieval pool or drawing matched replacements. Before retrieval, we remove the target interaction, matching source IDs, normalized exact duplicates, and embedding-near duplicates above a development-set threshold; consequently, neither retrieval nor matching can access an interaction or metadata recorded at or after the target. Matched other-user records come from the same task and legal partition and are matched using pre-action semantic content, length, retrieval-score bin, time bin, and available nonbehavioral metadata. Historical actions, target references, and all post-target fields are unavailable to the matcher. The matched-specificity test therefore estimates a frozen-LLM replacement effect and is not a causal claim about users.

\paragraph{Cost and statistical protocol.}
Profile cost is the actual incremental number of tokenizer tokens, not document count. We additionally record selected records, root \stopa{} rate, harm rate relative to the empty-profile output, generator calls, latency, and throughput. A root \stopa{} means \emph{no profile augmentation}; candidate retrieval has still occurred. Every reranker and generator-equivalent forward pass is included in the cost audit.

We choose the operating point, $\beta$, $\lambda$, loss weights, temperatures $T_r$ and $T_s$, calipers, $M$, $K_{\max}$, $K^-$, $J$, and $W$ using development data only; no test result is used for tuning, and all selected values are frozen before test evaluation. Each automatic result uses three complete pipeline seeds. We pair predictions at the instance level, use a paired 95\% bootstrap clustered by user rather than by isolated record, and apply Holm-corrected comparisons against all main baselines. Reported seed variation reflects variation across complete experimental runs.

\subsection{Generator-Scale Sensitivity}
\label{app:supplementary-results}

We evaluate task-generator capacity on the time-based (T) split by comparing dense Qwen3.5-4B, Qwen3.5-9B, and Qwen3.5-27B~\citep{qwen-team-2026-qwen35} generators while holding the retriever pool, search budget, task prompt, deterministic decoding, and leakage controls fixed. Tokenizer costs, frozen-generator likelihoods, bounded-tree labels, caches, and the controller are recomputed independently for each scale; no controller or generator-dependent artifact is transferred from the 9B condition.

Figure~\ref{fig:app-generator-scale} reports both official metrics for every task and all nine main methods. The 9B column reproduces the corresponding main-table results, while the 4B and 27B columns report independently measured scale conditions. Every cell aggregates three complete pipeline seeds. As in Table~\ref{tab:main-results}, every displayed metric is multiplied by 100.

\begin{figure*}[!t]
\centering
\includegraphics[width=.98\textwidth]{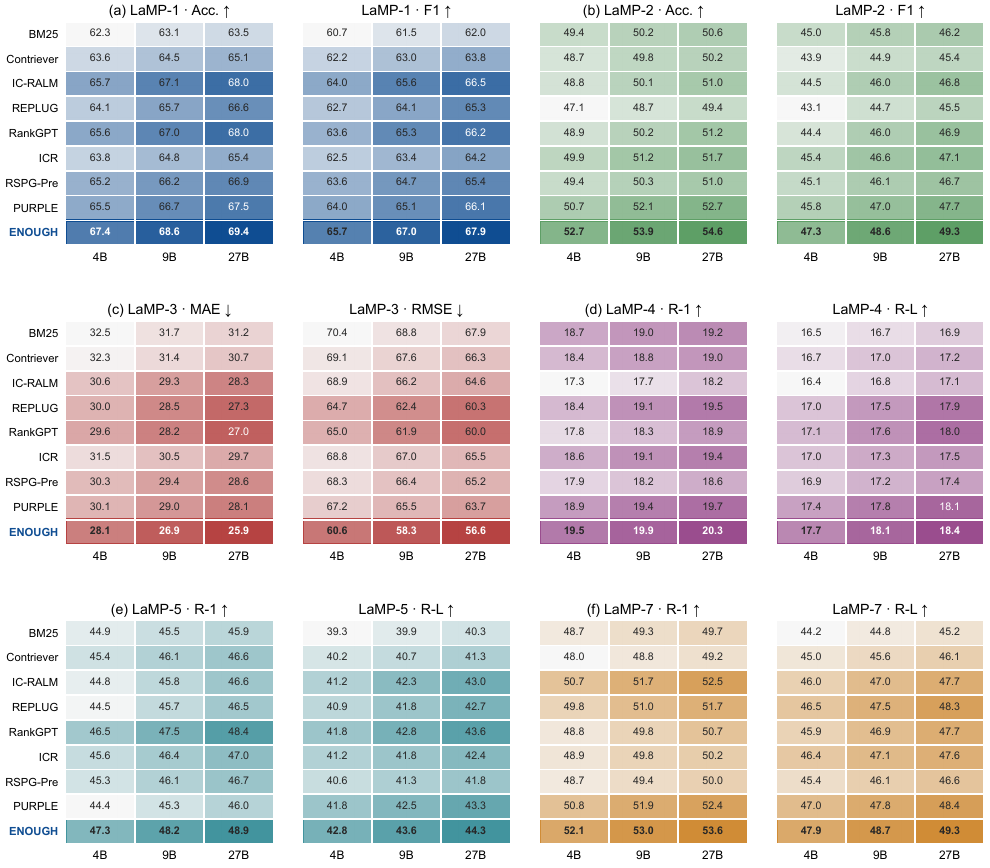}
\caption{Generator-scale sensitivity on the LaMP time-based (T) split, averaged over three complete pipeline seeds. Each task is represented by two adjacent panels, one per official metric, and task-level labels run from (a) through (f); each row therefore contains four panels for two tasks. We compare dense Qwen3.5-4B, Qwen3.5-9B, and Qwen3.5-27B task generators. Cells display raw values; darker shading denotes better performance within each panel after respecting metric direction, so color intensity is not comparable across panels. Higher is better except for MAE and RMSE.}
\label{fig:app-generator-scale}
\end{figure*}

The measured scale response differs across methods, tasks, and metrics rather than following a uniform offset. From 4B to 27B, baseline changes range from 1.2--2.5 points for classification accuracy, 1.2--2.6 for F1, 1.3--2.7 for MAE reduction, 2.5--5.0 for RMSE reduction, and 0.4--2.0 for ROUGE. The consistent gains show that greater generator capacity generally improves downstream quality, while the heterogeneous magnitudes indicate that scale does not affect every task or selector equally.
\FloatBarrier

\subsection{Profile Length and Stopping Diagnostics}
\label{app:stopping}

\paragraph{The diagnostics separate length effects from stopping quality.}
We report three complementary measured diagnostics. Table~\ref{tab:app-fixed-k} tests whether a common fixed length benefits all instances; Table~\ref{tab:app-forced-continuation} tests whether overriding a learned \stopa{} decision remains useful; and Figure~\ref{fig:app-stop-calibration} tests whether the learned stop score tracks its bounded-tree supervision. Unless stated otherwise, results use the time-based split with Qwen3.5-9B and average three complete pipeline seeds. Quality is the six-task macro-average after orienting each official task metric so that higher is better; token counts are actual increments over the empty-profile prompt.

\paragraph{Fixed profiles expose a quality--risk trade-off.}
We first vary the number of records while holding the frozen Contriever retriever, generator, and evaluation protocol fixed. A fixed-$k$ profile contains up to the first $k$ legal retrieved records.

As Table~\ref{tab:app-fixed-k} shows, the aggregate quality index rises from 96.15 with no profile to 98.71 at $k=5$, but falls to 97.93 at $k=10$. Over the same grid, mean profile cost reaches 1,755 tokens and the harm rate relative to the empty-profile output rises from the defined 0\% baseline to 31.7\%. Thus, quality is non-monotonic within the evaluated grid while both cost and instance-level risk increase with length. This diagnostic motivates instance-adaptive length selection; it does not identify $k=5$ as a universally optimal length.

\paragraph{Forced continuation isolates the decision value of \stopa{}.}
For each stopped prefix with enough remaining legal actions, we suppress \stopa{} and cumulatively append $d\in\{1,2,3\}$ records using the controller ordering, the frozen retriever ranking, or a reference-aware oracle-best diagnostic ordering. The $+d$ rows therefore evaluate cumulative continuations from the same prefix rather than the marginal effect of only the $d$th record; the oracle ordering is diagnostic and is not available at deployment. Relative to stopping at that prefix, outcomes are assigned in priority order: \emph{beneficial} if continuation increases $J_i^\lambda$; among the remaining cases, \emph{cost-not-worth} if direction-normalized task quality improves, \emph{harmful} if it worsens, and \emph{redundant} otherwise. This ordering makes the four categories mutually exclusive and exhaustive before rounding. Because $J_i^\lambda$ combines task gain, specificity, and token cost, \emph{cost-not-worth} is shorthand for a task-quality gain rejected by the full objective, not evidence that token cost alone caused the rejection.

\begin{table}[!t]
\centering
\small
\setlength{\tabcolsep}{3.0pt}
\begin{tabular}{lrrr}
\toprule
Length & Quality & Tokens & Harm \\
\midrule
k=0 & 96.15 & 0 & 0.0\% \\
k=1 & 97.56 & 200 & 13.2\% \\
k=3 & 98.45 & 560 & 18.6\% \\
k=5 & 98.71 & 918 & 23.2\% \\
k=10 & 97.93 & 1755 & 31.7\% \\
\bottomrule
\end{tabular}
\caption{Measured fixed-length diagnostic on the time-based split, averaged over three complete pipeline seeds. Profiles contain up to the first $k$ legal Contriever records. Quality is the direction-normalized six-task macro-index; tokens are incremental over the empty-profile prompt; harm is measured relative to the empty-profile output, so $k=0$ is the defined zero-harm baseline.}
\label{tab:app-fixed-k}
\end{table}

\begin{table*}[!t]
\centering
\small
\begin{tabular}{lcccccc}
\toprule
Continuation & Cum. depth & Cum. extra tok. & Beneficial & Redundant & Cost-not-worth & Harmful \\
\midrule
controller-next & +1 & 177 & 18.0\% & 36.2\% & 28.9\% & 16.9\% \\
controller-next & +2 & 351 & 14.2\% & 30.7\% & 29.1\% & 26.0\% \\
controller-next & +3 & 521 & 9.9\% & 27.1\% & 27.0\% & 35.9\% \\
\addlinespace
retriever-next & +1 & 181 & 14.0\% & 33.6\% & 29.0\% & 23.4\% \\
retriever-next & +2 & 359 & 9.7\% & 29.5\% & 28.0\% & 32.9\% \\
retriever-next & +3 & 536 & 7.2\% & 23.8\% & 25.0\% & 43.9\% \\
\addlinespace
oracle-best-continuation & +1 & 170 & 28.2\% & 33.6\% & 26.2\% & 11.9\% \\
oracle-best-continuation & +2 & 331 & 21.8\% & 30.7\% & 29.3\% & 18.2\% \\
oracle-best-continuation & +3 & 490 & 17.8\% & 29.1\% & 28.0\% & 25.1\% \\
\bottomrule
\end{tabular}
\caption{Measured forced-continuation diagnostic from prefixes where \stopa{} was predicted, averaged over three complete pipeline seeds. For each policy, $+d$ cumulatively appends its first $d$ records from the same prefix; extra tokens are cumulative relative to that prefix. Categories partition the evaluated stopped prefixes before rounding.}
\label{tab:app-forced-continuation}
\end{table*}

Table~\ref{tab:app-forced-continuation} shows that, for controller-ordered continuation, the beneficial share decreases from 18.0\% at $+1$ to 9.9\% at $+3$, whereas the harmful share increases from 16.9\% to 35.9\% as the cumulative added cost grows from 177 to 521 tokens. Retriever-ordered continuation is more harmful at every reported depth. The oracle-best diagnostic retains a larger beneficial share---28.2\% at $+1$---which confirms that some learned stops leave useful continuations, but its beneficial share also falls and its harmful share rises with depth. These are descriptive aggregate comparisons: without row-level paired intervals and a common-depth eligibility audit, they do not establish policy dominance or imply that every learned stop is optimal.

\begin{figure}[!t]
\centering
\includegraphics[width=.92\linewidth]{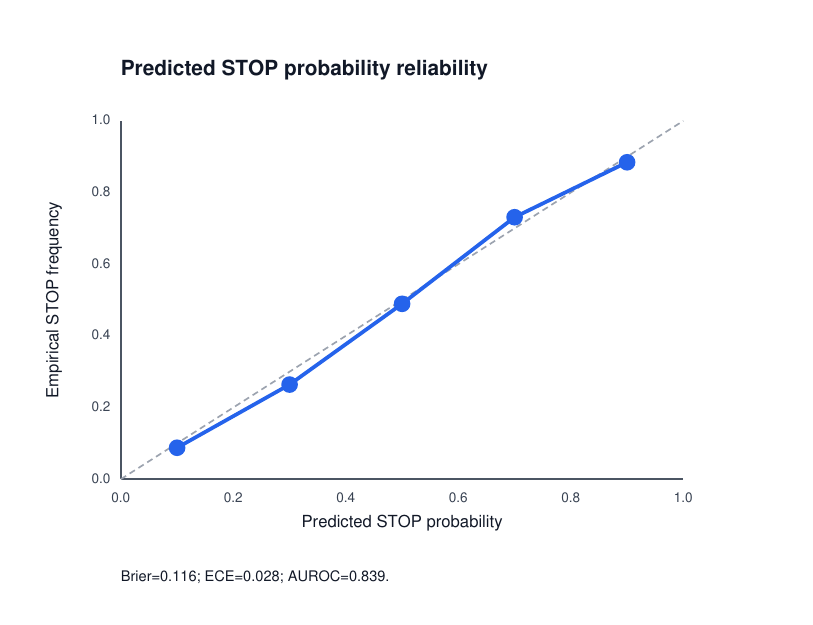}
\caption{Measured calibration of the temperature-scaled \stopa{} confidence against binary bounded-tree teacher decisions. Each point compares the mean predicted confidence with the teacher-label frequency in one of five bins; the dashed diagonal denotes perfect agreement. This is calibration to the bounded-tree teacher, not to exact global stopping optimality.}
\label{fig:app-stop-calibration}
\end{figure}

\paragraph{The stop confidence tracks the bounded-tree teacher, not an exact oracle.}
Figure~\ref{fig:app-stop-calibration} plots $\widehat p_{\mathrm{stop}}(s)=\sigma(\Delta_{\theta}^{\mathrm{stop}}(s)/T_s)$, where Eq.~\eqref{eq:stop-margin} defines the controller margin and Eq.~\eqref{eq:stop-loss} defines its temperature-scaled training target. The development-selected $T_s$ is frozen for test evaluation. We evaluate against the hard bounded-teacher decision $\mathbb{I}[\Delta_{\mathcal T_i}^{\mathrm{stop}}(s)\geq 0]$, consistent with resolving ties in favor of \stopa{}. The five mean confidence/teacher-frequency pairs are \PHStopCalibrationPairs{}. Within this binning, the Brier/ECE scores of \PHStopBrier{}/\PHStopECE{} indicate close aggregate agreement, while AUROC/AUPRC scores of \PHStopAUROC{}/\PHStopAUPRC{} indicate useful discrimination. These metrics do not estimate the probability that \stopa{} is globally optimal, and bin counts or uncertainty intervals are still needed to assess sparsely populated regions.

\paragraph{Residual stopping errors occur in both directions.}
The error audit records a stop as premature when the best available exact or bounded-teacher continuation value strictly exceeds the \stopa{} value, as late when the teacher preferred \stopa{} at an earlier prefix, and as forced when the controller reaches $K_{\max}$ without selecting \stopa{}. These are separate trajectory events and are not assumed to partition the diagnostic subset. Their respective rates are \PHPrematureStop{}, \PHLateStop{}, and \PHForcedKMax{}, and the associated mean utility loss relative to the audit's continuation benchmark is \PHStopRegret{}. The late-stop rate is numerically 1.9 percentage points higher than the premature-stop rate, but without paired uncertainty we treat this difference as descriptive rather than evidence of a systematic error asymmetry. Because this audit can use bounded-tree continuation values, its stopping-regret statistic should not be conflated with the reduced-space exact-oracle terminal-profile regret in Appendix~\ref{app:exact-oracle}. Overall, the diagnostics support \stopa{} as an informative, imperfect decision signal relative to its teacher; they do not establish globally minimal profiles.
\FloatBarrier

\subsection{Specificity, Matching, and Human Evaluation}
\label{app:specificity-audit}

We assess whether the selected profiles capture target-user behavior rather than merely providing useful task demonstrations from three complementary perspectives: a balance audit of the matched controls, frozen-generator interventions on profile ownership, action, and content, and blind human judgments of the resulting outputs.

\begin{table}[!t]
\centering
\small
\setlength{\tabcolsep}{2.0pt}
\begin{tabular}{lccccc}
\toprule
Audit & Coverage & Reject & Exact & \shortstack{Mean\\SMD} & \shortstack{p90\\SMD} \\
\midrule
\shortstack[l]{$K^-=3$;\\caliper enabled} & 87.5\% & 8.1\% & 31.1\% & 0.040 & 0.083 \\
\bottomrule
\end{tabular}
\caption{Measured balance audit for content-matched replacements. Coverage is the share of selected records with three valid controls; Reject is the share excluded by the caliper.}
\label{tab:app-matching}
\end{table}

\paragraph{Matching quality bounds the specificity analysis.}
Table~\ref{tab:app-matching} reports both the feasibility of constructing matched controls and their balance on variables observed before the user action. Coverage is the fraction of selected records for which all $K^-=3$ other-user controls satisfy the caliper; candidates that fail this requirement are rejected rather than silently replaced. Exact-match rate measures agreement within discrete matching strata, whereas standardized mean differences summarize balance over continuous covariates. Because the subsequent replacement comparisons exclude unmatched records, they characterize the matched subset rather than all selected profiles.

\paragraph{Interventions distinguish personal from generic utility.}
Table~\ref{tab:specificity-interventions} compares the target-user profile with five controls. Content-matched replacement approximately preserves observable content and query relevance while replacing a target-user record with a record from another user; random-user replacement and whole-profile swaps provide progressively less constrained controls. Action shuffling breaks the pairing between a historical context and the action taken in that context. Finally, owner permutation keeps record contents fixed while reassigning ownership labels, serving as a negative-control audit of whether the matching pipeline itself can manufacture an apparent specificity effect. These comparisons estimate frozen-generator replacement effects conditional on the achieved covariate balance; they do not identify a causal effect of user identity.

\begin{table}[!t]
\centering
\small
\setlength{\tabcolsep}{2.5pt}
\begin{tabular}{llrrrr}
\toprule
Task & Criterion & Win & Tie & Loss & $\kappa$ \\
\midrule
LaMP-4 & task faithfulness & 56.0\% & 24.0\% & 20.0\% & 0.41 \\
LaMP-4 & user style consistency & 61.3\% & 21.0\% & 17.7\% & 0.47 \\
LaMP-4 & overall preference & 58.7\% & 23.0\% & 18.3\% & 0.44 \\
\addlinespace
LaMP-7 & task faithfulness & 53.7\% & 25.3\% & 21.0\% & 0.41 \\
LaMP-7 & user style consistency & 60.3\% & 22.0\% & 17.7\% & 0.44 \\
LaMP-7 & overall preference & 58.7\% & 24.0\% & 17.3\% & 0.42 \\
\bottomrule
\end{tabular}
\caption{Measured three-annotator blind comparison of \method{} against PURPLE on 200 samples per task. Win, tie, and loss are defined from the perspective of \method{}; $\kappa$ denotes Fleiss' agreement coefficient.}
\label{tab:app-human}
\end{table}

\paragraph{Blind human judgments test whether the differences are perceptible.}
For each of LaMP-4 and LaMP-7, Table~\ref{tab:app-human} compares \method{} with PURPLE on 200 examples, with three annotators independently judging each example. Task faithfulness, consistency with the user's style, and overall preference are elicited as separate questions so that task correctness is not conflated with stylistic personalization. Presentation order is randomized and method identity is concealed. Annotators see the current input and the legally available user history, but not likelihood scores; automatic references are also hidden for the overall-preference judgment. Ties are retained as valid outcomes, and adjudication is restricted to protocol violations. The table reports win/tie/loss rates and Fleiss' $\kappa$ from the completed annotation study.
\FloatBarrier

\subsection{Additional Ablations}
\label{app:ablations-oracle}

The main ablation figure isolates long-horizon targets, adaptive stopping, user specificity, and token cost. We additionally test the value-head decomposition, ranking and STOP losses, and order-sensitive profile encoding.

\begin{table}[!t]
\centering
\small
\setlength{\tabcolsep}{3.0pt}
\begin{tabular}{lrrrr}
\toprule
Variant & Quality & Tokens & Harm & Regret \\
\midrule
Full \method{} & 99.9 & 512 & 15.0\% & 0.070 \\
Single net-value head & 99.4 & 515 & 17.1\% & 0.096 \\
No ranking loss & 99.4 & 516 & 16.7\% & 0.090 \\
No STOP loss & 99.0 & 581 & 18.7\% & 0.121 \\
Permutation-invariant profile & 99.2 & 513 & 17.5\% & 0.112 \\
\bottomrule
\end{tabular}
\caption{Measured controller-objective and representation ablations, averaged over three complete pipeline seeds. Higher Quality and lower token cost, harm, and regret are better.}
\label{tab:app-extra-ablation}
\end{table}

\FloatBarrier

\subsection{Exact Minimal-Sufficient Oracle}
\label{app:exact-oracle}

\begin{table}[!t]
\centering
\small
\setlength{\tabcolsep}{2.6pt}
\begin{tabular}{lrrrr}
\toprule
Selector & Regret & $\epsilon$-suff. & Excess tok. & Stop agree \\
\midrule
Exact oracle & 0.000 & 100.0\% & 0 & 100.0\% \\
Bounded-tree teacher & 0.028 & 91.6\% & 41 & 92.3\% \\
Learned \method{} & 0.069 & 81.8\% & 89 & 83.7\% \\
Myopic & 0.129 & 68.5\% & 169 & 71.9\% \\
Fixed $k=4$ & 0.152 & 64.1\% & 248 & 65.2\% \\
\bottomrule
\end{tabular}
\caption{Measured exact-oracle diagnostic pooled over 50 examples per task with $M'=6$ and $K'=4$. Lower is better for regret and excess tokens; higher is better for $\epsilon$-sufficiency and stop agreement.}
\label{tab:oracle}
\end{table}

\begin{table}[!t]
\centering
\small
\begin{tabular}{lrrrr}
\toprule
Task & Teacher & \method{} & Myopic & Fixed-$4$ \\
\midrule
LaMP-1 & 0.025 & 0.064 & 0.122 & 0.147 \\
LaMP-2 & 0.030 & 0.073 & 0.130 & 0.154 \\
LaMP-3 & 0.042 & 0.082 & 0.141 & 0.162 \\
LaMP-4 & 0.024 & 0.066 & 0.126 & 0.148 \\
LaMP-5 & 0.037 & 0.080 & 0.137 & 0.161 \\
LaMP-7 & 0.023 & 0.065 & 0.121 & 0.144 \\
\bottomrule
\end{tabular}
\caption{Measured utility regret to the exact oracle by task; lower is better, and the exact oracle has zero regret by definition.}
\label{tab:app-oracle-task}
\end{table}

We measure how closely each selector approaches a minimal sufficient profile using an exact-oracle diagnostic on 50 examples per task. For each instance, we reduce the candidate pool to $M'=6$, cap the profile length at $K'=4$, and enumerate all 517 legal ordered profiles: the empty profile and every length-one-to-four permutation drawn without replacement. All methods use the same frozen generator and exact token cost. Because the exact oracle scores profiles using the gold reference, it provides a diagnostic upper bound rather than a deployable selector.

The bounded-tree teacher retains most of the oracle performance, with 0.028 utility regret, 91.6\% $\epsilon$-sufficiency, 41 excess tokens, and 92.3\% agreement with the oracle stop decision (Table~\ref{tab:oracle}). The learned \method{} controller obtains 0.069 regret and reaches $\epsilon$-sufficiency on 81.8\% of instances while using 89 excess tokens and matching the oracle stop decision 83.7\% of the time. It is substantially closer to the oracle than the myopic and fixed-$k$ selectors on all four diagnostics. The gap between the exact oracle and bounded teacher quantifies bounded-search truncation, whereas the additional gap from the teacher to \method{} reflects controller approximation.

The task-level results show the same ordering: \method{} regret ranges from 0.064 on LaMP-1 to 0.082 on LaMP-3, below both the myopic and fixed-$4$ alternatives on every task (Table~\ref{tab:app-oracle-task}). Pooled quantities are computed directly over diagnostic instances rather than by averaging heterogeneous task-level metrics.
\FloatBarrier

\subsection{Additional Hyperparameter Sensitivity}
\label{app:additional-hyperparameter-sensitivity}

To assess whether \method{} depends on a narrow configuration, we conduct additional one-at-a-time sweeps on the development set over the candidate-pool size $M$, maximum profile depth $K_{\max}$, number of matched controls $K^{-}$, offline search budget $(J,W)$, and auxiliary-loss weights $\eta$ and $\zeta$. All remaining settings are held fixed at their selected values. For each task, we orient its official development metric so that higher is better and normalize it to a common 100-point reference level; the reported quality index is the macro-average of the six task-level indices rather than an average of raw heterogeneous metrics. Stop regret is the mean utility shortfall of the controller-selected stopping point relative to the best feasible stopping value on the exact-oracle diagnostic subset in Appendix~\ref{app:exact-oracle}. For the structural and search parameters, we report both metrics; for the loss weights, we report only aggregate quality. Following the protocol above, each plotted point aggregates three complete pipeline seeds. Because the figures do not show seed-level intervals, we interpret the trends descriptively within the evaluated grids and do not use these sweeps to claim statistical significance. All operating points are selected using development data and frozen before test evaluation.

\begin{figure*}[!t]
\centering
\includegraphics[width=.98\textwidth]{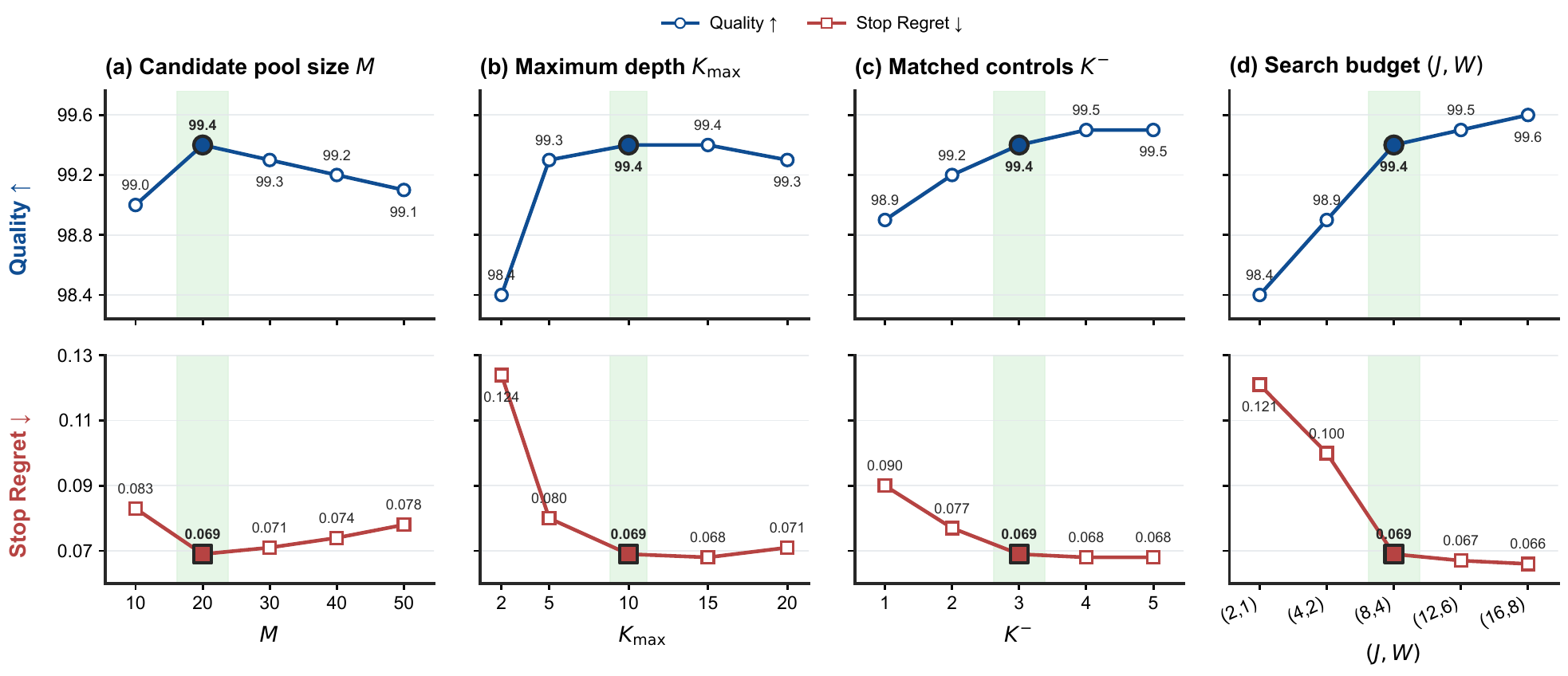}
\caption{Measured development-set sensitivity of \method{} to candidate-pool size $M$, maximum profile depth $K_{\max}$, number of matched controls $K^{-}$, and offline search budget $(J,W)$. The upper row reports the aggregate quality index (higher is better), and the lower row reports stop regret (lower is better). Shaded bands and black-outlined markers denote the selected settings $M=20$, $K_{\max}=10$, $K^{-}=3$, and $(J,W)=(8,4)$.}
\label{fig:app-structural-hyperparameters}
\end{figure*}

\paragraph{Structural and search parameters.}
Figure~\ref{fig:app-structural-hyperparameters} shows that the selected settings lie near favorable regions of their respective sweeps. Among the evaluated pool sizes, $M=20$ attains both the highest quality value, 99.4, and the lowest stop regret, 0.069; enlarging the pool further produces small degradations in both metrics. Increasing $K_{\max}$ from 2 to 10 substantially improves quality and stop regret, whereas $K_{\max}=15$ preserves the same quality of 99.4 and lowers regret by only 0.001. Similarly, $K^{-}=3$ captures most of the improvement from additional matched controls: moving to four or five controls changes quality by at most 0.1 and stop regret by at most 0.001. Larger search budgets continue to improve both metrics, but the gains diminish beyond $(J,W)=(8,4)$: increasing the budget to $(16,8)$ raises quality from 99.4 to 99.6 and reduces stop regret from 0.069 to 0.066. These trends motivate moderate operating points that retain strong development-set performance while bounding candidate processing, matched-control scoring, and offline search.

\begin{figure}[!t]
\centering
\includegraphics[width=\columnwidth]{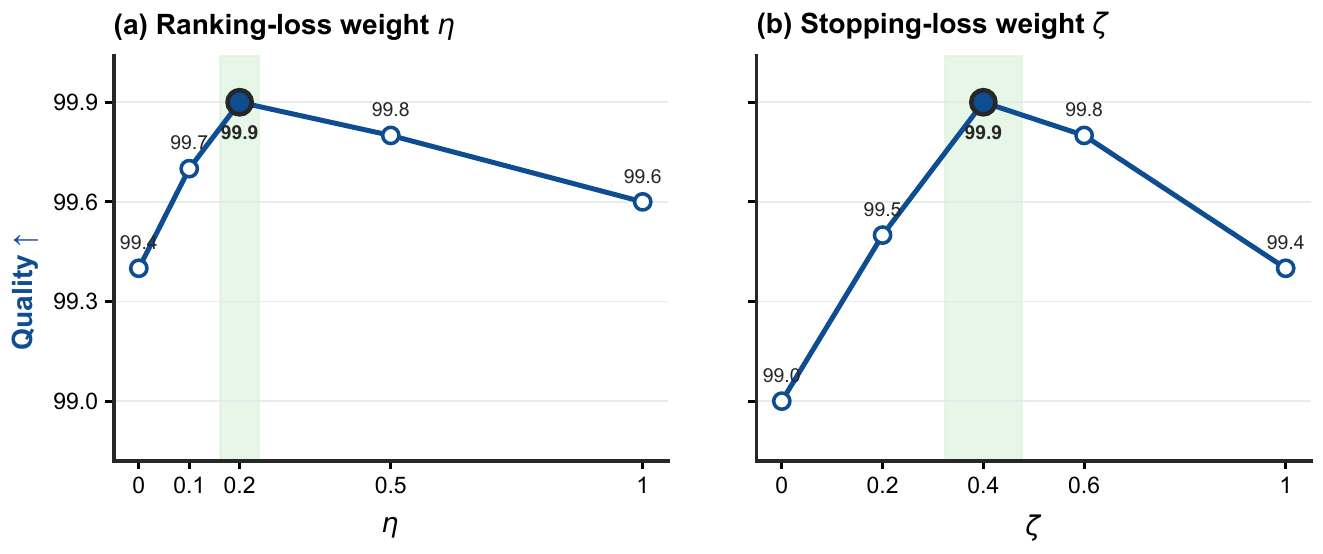}
\caption{Measured development-set quality sensitivity to the ranking-loss weight $\eta$ and stopping-loss weight $\zeta$ in Eq.~\eqref{eq:total-loss}. Each weight is varied within $[0,1]$ while the remaining configuration is held fixed. Higher is better; shaded bands and black-outlined markers denote the selected values $\eta=0.2$ and $\zeta=0.4$.}
\label{fig:app-loss-weight-sensitivity}
\end{figure}

\paragraph{Training-loss weights.}
Both loss-weight sweeps favor moderate, nonzero coefficients within the evaluated grids. The selected value $\eta=0.2$ reaches a quality index of 99.9, compared with 99.4 when the ranking loss is omitted and 99.6 at $\eta=1$. Likewise, $\zeta=0.4$ reaches 99.9, compared with 99.0 at $\zeta=0$ and 99.4 at $\zeta=1$. The neighboring settings remain competitive---quality is 99.8 at both $\eta=0.5$ and $\zeta=0.6$---so the results indicate a favorable region rather than a sharply isolated optimum. The sweeps therefore support balancing the ranking and stopping objectives with value regression. Because Figure~\ref{fig:app-loss-weight-sensitivity} reports only aggregate quality, it does not by itself establish how either coefficient affects stopping calibration, profile length, or task-level variability.
\FloatBarrier

\subsection{Robustness and Subgroups}
\label{app:robustness-efficiency}

\paragraph{Controlled stress tests.}
Starting from the one-useful-record condition, a near duplicate changes selected length by \PHRobustDuplicateLengthDelta{} records. The action-conflicting condition changes harm from \PHRobustUsefulHarm{} to \PHRobustConflictHarm{} and quality by \PHRobustConflictQualityDrop{} index points (Table~\ref{tab:app-robustness}). The subgroup analysis stratifies history length, match coverage, and input difficulty.

\begin{table}[!t]
\centering
\small
\setlength{\tabcolsep}{2.0pt}
\begin{tabular}{lrrrrr}
\toprule
Candidate ladder & Quality & Tokens & Records & \shortstack{Root\\\stopa{}} & Harm \\
\midrule
No useful history & 96.3 & 0 & 0.00 & 100.0\% & 0.0\% \\
+ useful record & 99.3 & 220 & 0.99 & 18.2\% & 12.0\% \\
+ near duplicate & 99.3 & 233 & 1.06 & 17.7\% & 11.9\% \\
+ irrelevant record & 99.2 & 237 & 1.08 & 17.0\% & 13.1\% \\
+ action conflict & 98.5 & 266 & 1.25 & 15.0\% & 19.0\% \\
\bottomrule
\end{tabular}
\caption{Measured controlled candidate-injection stress test. Each row adds one perturbation to the preceding candidate-pool condition.}
\label{tab:app-robustness}
\end{table}

\begin{table}[!t]
\centering
\small
\setlength{\tabcolsep}{3.0pt}
\begin{tabular}{llrrrr}
\toprule
Dimension & Bucket & Quality & Tokens & Root \stopa{} & Harm \\
\midrule
History & short & 98.9 & 302 & 26.6\% & 13.7\% \\
History & medium & 99.5 & 506 & 16.8\% & 15.2\% \\
History & long & 99.6 & 727 & 8.3\% & 18.0\% \\
\addlinespace
Coverage & low & 98.5 & 354 & 30.1\% & 18.2\% \\
Coverage & medium & 99.3 & 501 & 16.9\% & 15.4\% \\
Coverage & high & 99.7 & 656 & 8.1\% & 13.8\% \\
\addlinespace
Difficulty & easy & 99.5 & 425 & 20.2\% & 11.2\% \\
Difficulty & medium & 99.3 & 514 & 16.6\% & 14.9\% \\
Difficulty & hard & 98.8 & 611 & 13.1\% & 22.0\% \\
\bottomrule
\end{tabular}
\caption{Measured subgroup analysis by user-history length, matched-candidate coverage, and input difficulty. Buckets are defined by development-set quantiles and frozen for test evaluation.}
\label{tab:app-subgroups}
\end{table}

The robustness ladder inserts one perturbation at a time into an otherwise controlled pool. Near duplicates test redundancy, random records test distraction, and action conflicts test whether the controller can reject semantically relevant but behaviorally inconsistent evidence. Subgroups are defined on development-set quantiles and frozen before test evaluation.
\FloatBarrier

\begin{figure}[!t]
\centering
\includegraphics[width=0.75\columnwidth]{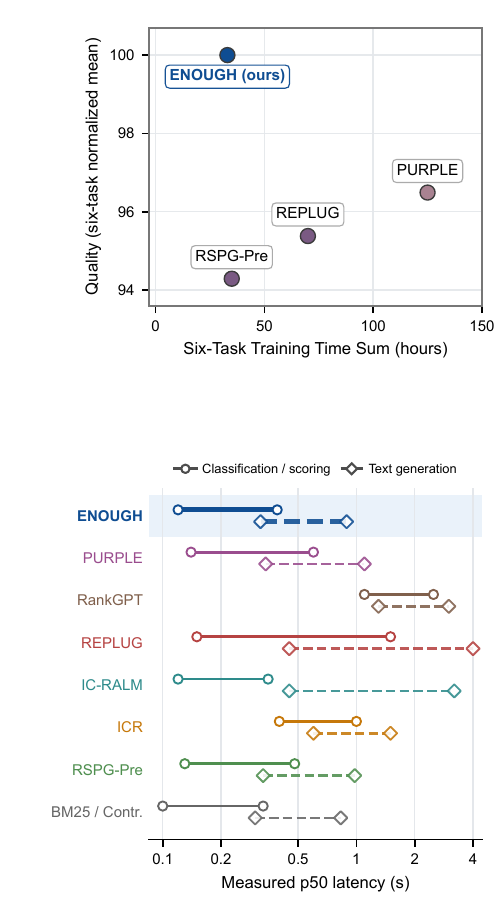}
\caption{Detailed training-time and latency comparison. Top: Quality versus the measured six-task training-time sum on one A800 for the four methods that require task-specific GPU training; Quality is the equal-weight mean of six task-normalized scores. Bottom: measured batch-size-one p50 latency ranges across classification/scoring tasks (solid lines with circles) and text-generation tasks (dashed lines with diamonds); BM25 and Contriever are grouped in this panel.}
\label{fig:app-quality-training-latency}
\end{figure}

\subsection{Online and Offline Efficiency}
\label{app:efficiency}

\begin{table}[!t]
\centering
\small
\begin{tabular}{lrrrrr}
\toprule
Method & Quality & Tokens & Calls & p50 ms & p95 ms \\
\midrule
\method{} & 100.00 & 510 & 1 & 370 & 440 \\
PURPLE & 96.49 & 905 & 1 & 398 & 474 \\
RankGPT & 96.32 & 869 & 6 & 1120 & 1379 \\
REPLUG & 95.38 & 1388 & 5 & 1323 & 1595 \\
IC-RALM & 94.82 & 1210 & 7 & 1463 & 1793 \\
ICR & 94.30 & 819 & 4 & 865 & 1049 \\
RSPG-Pre & 94.29 & 765 & 3 & 796 & 951 \\
Contriever & 92.17 & 931 & 1 & 365 & 442 \\
BM25 & 90.96 & 905 & 1 & 360 & 421 \\
\bottomrule
\end{tabular}
\caption{Measured online inference cost under identical generation settings. Tokens is the mean selected-profile length; Calls is the number of generator-equivalent invocations; p50/p95 are end-to-end latency in milliseconds.}
\label{\efficiencytablelabel}
\end{table}

\begin{table}[!t]
\centering
\small
\setlength{\tabcolsep}{3.0pt}
\begin{tabular}{lrrrrr}
\toprule
Method & Retrieval & Controller & Selector & Generator & Other \\
\midrule
Contriever-$5$ & 17 & -- & 1 & 362 & 15 \\
PURPLE-$5$ & 18 & -- & 10 & 357 & 16 \\
\method{} & 17 & 11 & -- & 325 & 17 \\
\bottomrule
\end{tabular}
\caption{Measured p50 end-to-end latency decomposition in milliseconds under the common batch-size-one evaluation protocol.}
\label{tab:app-latency-components}
\end{table}

\begin{table}[!t]
\centering
\small
\setlength{\tabcolsep}{2.0pt}
\begin{tabular}{ccccc}
\toprule
\shortstack{Teacher\\calls (M)} & \shortstack{Cache\\hit} & \shortstack{GPU-\\hours} & \shortstack{Storage\\(GB)} & \shortstack{Extra GPU-h\\per $\lambda$} \\
\midrule
9.7 & 74.8\% & 336 & 95 & 95 \\
\bottomrule
\end{tabular}
\caption{Measured offline teacher-scoring and cache cost for one operating point. Teacher calls are reported in millions.}
\label{tab:app-offline-cost}
\end{table}

All online timings use the same single-GPU machine, batch size one, warmed caches, and identical generator settings. We report median and p95 end-to-end latency and decompose the median into retrieval, controller or selector, generation, and residual preprocessing. Multi-call baselines include every reranking or generation call. Offline teacher scoring, cache construction, and additional $\lambda$ operating points are excluded from online latency but reported explicitly in Table~\ref{tab:app-offline-cost}.

\paragraph{Training and task-family detail.}
Figure~\ref{fig:app-quality-training-latency} separates offline training time from task-dependent online latency. \method{} reaches Quality 100.00 with a six-task training sum of 33 hours, close to RSPG-Pre at 35 hours and below REPLUG and PURPLE at 70 and 125 hours. In the measured latency comparison, \method{} remains in the low-latency group for both task families; BM25/Contriever have slightly lower ranges, whereas RankGPT and the generation paths of REPLUG and IC-RALM extend to substantially higher upper bounds. Because the two panels measure different cost dimensions, we do not collapse them into a single efficiency score.

\FloatBarrier

\section{Complete Generator Prompt Templates}
\label{app:prompt-templates}

Figure~\ref{fig:prompt-templates} shows the complete logical system and user messages supplied to the frozen generator before model-specific chat serialization. The gray system instruction and the \texttt{<TASK>}, \texttt{<USER\_HISTORY>}, and \texttt{<CURRENT\_REQUEST>} delimiters are shared by every method. Within \texttt{<CURRENT\_REQUEST>}, we retain the official LaMP task instruction verbatim and replace only instance-specific content with descriptive braced fields. The gold block uses our uniform behavioral-record serialization; records are repeated in the order chosen by the selector. If the controller selects \stopa{} at the root, the two \texttt{<USER\_HISTORY>} delimiters remain but enclose no record. Display line breaks inside the blue block are visual wrapping only. The official Qwen3.5-9B and Llama-3.1-8B-Instruct chat templates subsequently add their model-specific control tokens.

For LaMP-1--5, the historical context and user action correspond to the task fields reported in each figure. Public LaMP-7 histories contain only the user's prior tweet and no separate pre-action context; accordingly, its \texttt{[HIST\_CONTEXT]} field is empty and the tweet is serialized as \texttt{[USER\_ACTION]}. This is a property of the released benchmark schema rather than a task-specific prompt exception.

\begin{figure*}[p]
\centering
\includegraphics[width=\textwidth]{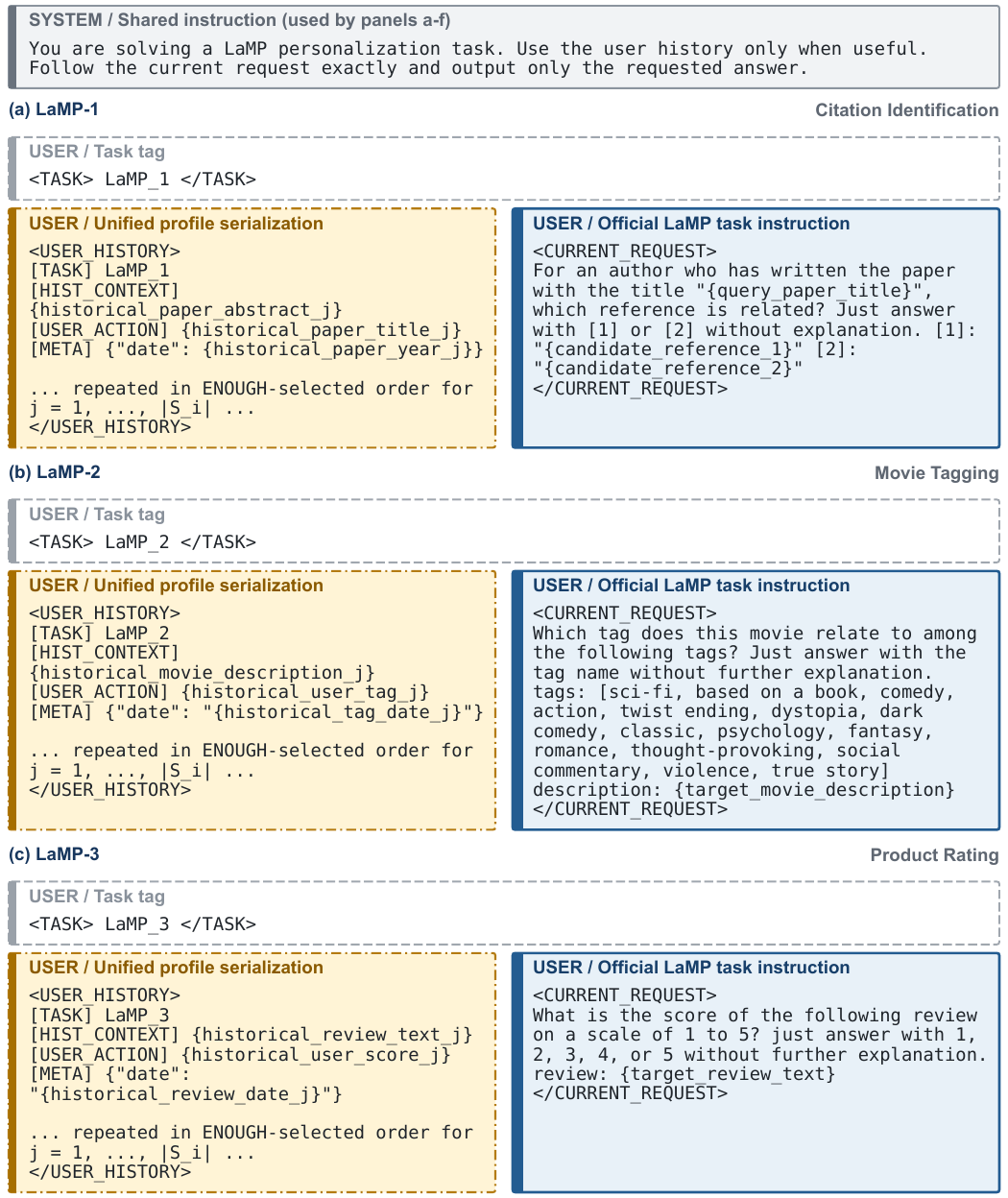}
\end{figure*}
\clearpage

\begin{figure*}[p]
\centering
\includegraphics[width=\textwidth]{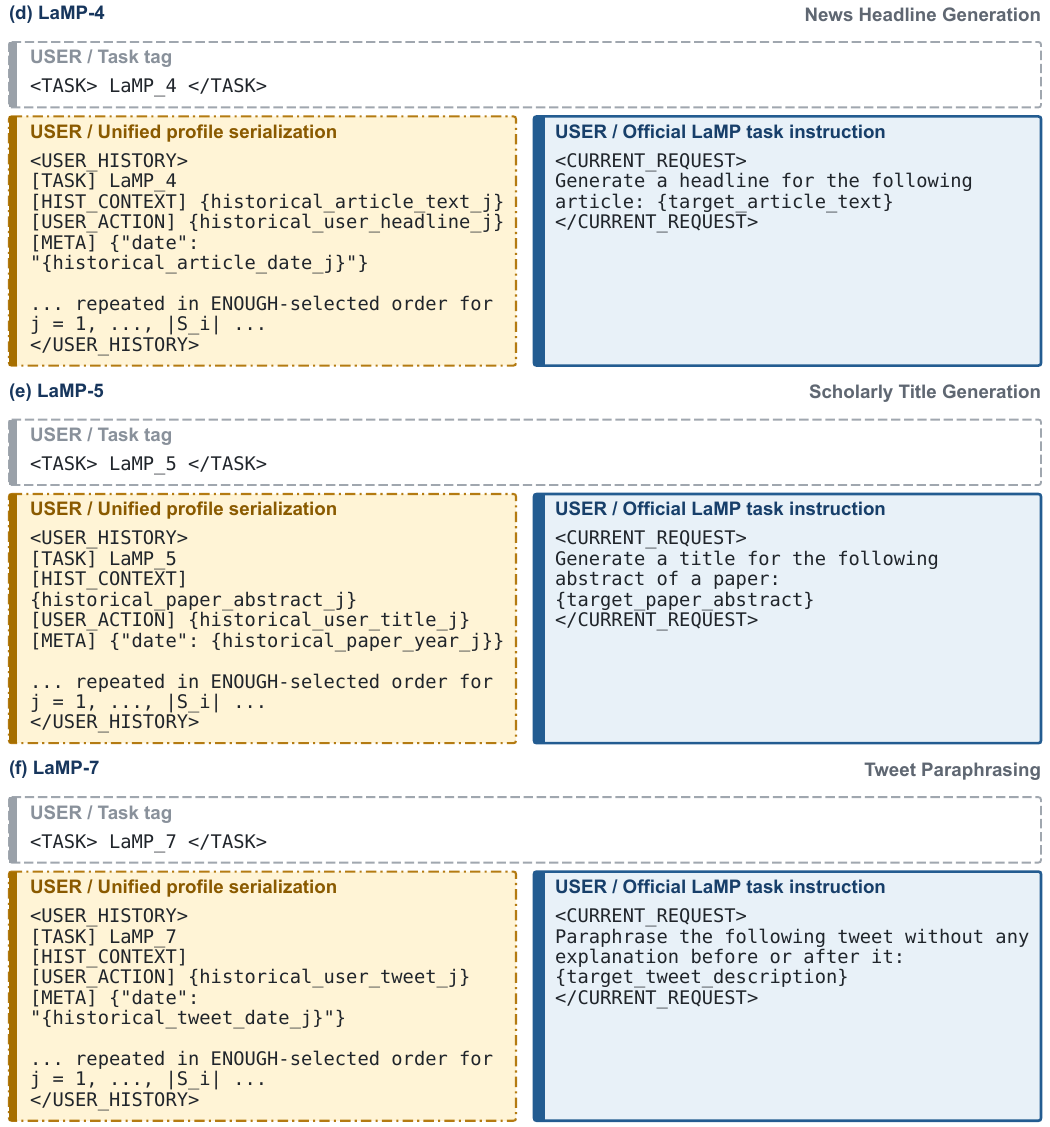}
\caption{Complete logical generator-prompt templates for (a) LaMP-1 personalized citation identification, (b) LaMP-2 personalized movie tagging, (c) LaMP-3 personalized product rating, (d) LaMP-4 personalized news headline generation, (e) LaMP-5 personalized scholarly title generation, and (f) LaMP-7 personalized tweet paraphrasing. The system instruction shown on the first page is shared by all six templates.}
\label{fig:prompt-templates}
\end{figure*}
\FloatBarrier

\end{document}